\documentclass[11pt]{article} 

\usepackage{amsmath}
\usepackage{amstext}
\usepackage{amssymb}
\usepackage{amsfonts}
\usepackage{color}
\usepackage{algorithm}
\usepackage{algorithmic}

\usepackage[pdftex]{graphicx}

\newtheorem{theorem}{Theorem}[section]
\newtheorem{lemma}[theorem]{Lemma}

\newtheorem{corollary}[theorem]{Corollary}

\newenvironment{proof}[1][Proof]{\textbf{#1.} }{\ \rule{0.5em}{0.5em}}

\begin{document}

\title{\textbf{Video Compressive Sensing for Dynamic MRI}}
\author{Jianing V. Shi$^{1,2}$\thanks{Corresponding author's email address: jianing@math.ucla.edu}, Wotao Yin$^{2}$, Aswin C. Sankaranarayanan$^{3}$, and Richard G. Baraniuk$^{1}$\\
\small{$^{1}$ Department of Electrical and Computer Engineering, Rice University}\\
\small{$^{2}$ Department of Mathematics, UCLA}\\
\small{$^{3}$ Department of Electrical and Computer Engineering, Carnegie Mellon University}
 }
\date{\today}
\maketitle

\begin{abstract}

%Learning compact representation of high-dimensional data with motion textures and computation on motion manifold is a task central to computer vision and machine learning.  We consider a specific type of motion manifold as linear dynamical systems, and present estimation of manifold parameters and computation on the motion manifold in the context of video compressive sensing.  Given compressive measurements, the state sequence can be first estimated by assuming a low rank model.  We then reconstruct the video on its motion manifold using a joint structured sparsity assumption.  In particular, we minimize an objective function with a mixture of wavelet sparsity and total variation, along with simultaneous sparsity among columns of the observation matrix.  Our main contribution is to derive an efficient convex optimization algorithm through alternating direction augmented Lagrangian method.  We demonstrate the performance of our approach for video compressive sensing, in terms of reconstruction accuracy and simulated phase transition.

We present a video compressive sensing framework, termed kt-CSLDS, to accelerate the image acquisition process of dynamic magnetic resonance imaging (MRI).  We are inspired by a state-of-the-art model for video compressive sensing that utilizes a linear dynamical system (LDS) to model the motion.  Given compressive measurements, the state sequence of an LDS can be first estimated using system identification techniques.  We then reconstruct the observation matrix using a joint structured sparsity assumption.  In particular, we minimize an objective function with a mixture of wavelet sparsity and joint sparsity within the observation matrix.  We derive an efficient convex optimization algorithm through alternating direction method of multipliers (ADMM), and provide a theoretical guarantee for global convergence. We demonstrate the performance of our approach for video compressive sensing, in terms of reconstruction accuracy.  We also investigate the impact of various sampling strategies.  We apply this framework to accelerate the acquisition process of dynamic MRI and show it achieves the best reconstruction accuracy with the least computational time compared with existing algorithms in the literature.
\end{abstract}

\section{Introduction}

Our fascination with detail has lead to sensors of ever increasing capabilities.  In many modalities, the traditional sampling theory associated with Nyquist sampling theorem fails to deliver high spatio-temporal resolution or at best, delivers this at prohibitive costs. Compressive sensing \cite{candes2006a, candes2006b, donoho2006} has recently arisen as a paradigm to revolutionize the design of sensors and signal processing \cite{baraniuk2011}, and is promised to deliver better sensors.  The key insight of compressive sensing is that one can design a sensing system that only acquires a few linear measurements and recover the signals via convex optimization or greedy pursuit.  Two assumptions are essential to the success of compressive sensing: a) the signal can be approximated using sparse representation under a suitable basis or dictionary; b) linear measurements are suitably incoherent with the basis or dictionary in which the signal is represented.  Compressive sensing enables one to acquire highly undersampled data during the acquisition process, with the sampling rate way below the Nyquist sampling frequency.

A prominent application of compressive sensing is to accelerate the acquisition process of magnetic resonance imaging (MRI).  In fact, the discovery of compressive sensing was largely motivated by the MRI problem \cite{candes2006a}, where one wishes to reconstruct an object based on incomplete Fourier samples.  The physical mechanism of MRI \cite{damadian1971, lauterbur1973, mansfield1977} requires scanning in the Fourier space in order to reconstruct an object.  The speed of imaging is fundamentally limited by physical constraints such as gradient amplitude and slew rate, as well as physiological constraints \cite{lustig2007}.  Compressive sensing has proven to be very successful in accelerating the acquisition process of MRI and has opened up many possibilities for new clinical applications \cite{lustig2007, lustig2008}.  Dynamic MRI reconstructs a dynamic sequence of images based on measurements in spatial frequency versus time (k-t) domain, roughly speaking, video acquired in the Fourier space.  As pointed out in \cite{lustig2008}, dynamic MRI is challenging due to the time-varying nature of the imaging object and the spatio-temporal tradeoff.  Moreover, the discrepancy between the dynamic nature of the moving object and a static scene assumption for sensing creates a spatio-temporal recovery error \cite{sankaranarayanan2012, park2013}. To tackle such a challenge, we leverage ideas from state-of-the-art video compressive sensing frameworks.

Video compressive sensing is nontrivial due to its high-dimensional representation and the ephemeral nature of videos.  In order to achieve recovery using as few samples as possible, it requires one to exploit the redundancy in the ambient space and go beyond a frame-by-frame reconstruction \cite{wakin2006}.  The general philosophy is to identify a model in which signal can be represented parsimoniously, and identify the basis or dictionary to sparsify the signal under the signal model.  Several signal models have been employed to perform video compressive sensing.  One early approach considered sparse representation in both spatial and temporal domains by treating video as a three-dimensional matrix, and employed 3D wavelet transform to sparsify the video \cite{wakin2006}.  A later approach took advantage of the small inter-frame differences together with spatial 2D wavelet transform within each frame, which was implemented in the compressive coded aperture video camera \cite{marcia2008}.  Further work along the video coding ideas sought to reconstruct individual frames based on wavelet sparsity in the spatial domain, while modeling temporal dependencies between frames using motion compensation methods, such as lifting based wavelet sparsity \cite{secker2003} and optical flow \cite{reddy2011}.  Multi-scale recovery algorithms along with various motion compensation mechanisms were investigated in \cite{park2009, sankaranarayanan2012, park2013}.  Work based on the separation of background and moving objects were investigated in \cite{cevher2008}.  Video compressive sensing models based on the entire image typically involve a dense measurement matrix, which is computationally expensive, therefore block-based video compressive sensing \cite{fowler2012} divided each frame into small blocks in order to accelerate computation.  Dictionary learning-based methods were proposed to identify task specific basis for compressive sensing reconstruction \cite{prades2009, chen2010, rajwade2012}.  Another approach of exploiting the redundancy was to consider motion manifold and build a global model for the video cube.  The key idea is to project the original video cube onto a motion manifold, and perform reconstruction within a low-dimensional space.  Motion manifold models arise in many computer vision and machine learning problems, such as dynamical textures modeling \cite{doretto2003, chan2008}, human activity tracking \cite{bissacco2001, veeraraghavan2005}, video-based face recognition \cite{aggarwal2004}, data-driven motion synthesis \cite{lee2006}, video compressive sensing \cite{sankaranarayanan2010}, coded strobing photography \cite{veeraraghavan2011}.  A key promise of the motion manifold model is to obtain a compact representation of high-dimensional data by exploring the spatio-temporal structures, hence enabling computation on the low-dimensional manifold.

Prior work on compressively sensed dynamic MRI include k-t SPARSE \cite{lustig2006}, k-t FOCUSS \cite{jung2009}, Modified CS \cite{lu2009}, MASTeR \cite{asif2012}, Subtraction Sparsity \cite{rapacchi2013}, and L$+$S reconstruction \cite{otazo2013}.  k-t SPARSE and k-t FOCUSS both use the wavelet transform to model sparsity in the temporal domain.  Modified CS identifies signal support in the first frame to facilitate reconstruction of the rest video frames using compressive sensing.  MASTeR uses motion adaptive spatio-temporal regularization to perform reconstruction based on compressive k-t data.  Subtraction Sparsity is designed for contrast-enhanced magnetic resonance angiography, which subtracts a pre-contrast mask from all post-contrast frames to promote sparsity in the resulting difference images. L$+$S reconstruction uses a low rank and sparse matrix decomposition to separate background and dynamic components.

In this paper, we have chosen to build upon CS-LDS \cite{sankaranarayanan2010}, a video compressive sensing approach for time varying signals modeled as a linear dynamical system (LDS).  Encouraged by its high-fidelity reconstruction quality for a wide arrange of videos and its ability to achieve high compression rates, we extend the CS-LDS model to the k-t domain. We propose a novel and efficient algorithm, which we call kt-CSLDS.  Our proposed algorithm takes advantage of the orthonormal property of the Fourier operator and uses a number of numerical techniques to achieve high computational efficiency.  We provide theoretical guarantee for its global convergence.  We use the kt-CSLDS model to accelerate the image acquisition process of dynamic MRI.  Finally, we investigate the impact of sampling strategies on the reconstruction quality of dynamic MRI.

\section{Compressive Sensing Dynamic MRI Model}

Learning a low-dimensional signal model based on video data is an important topic in computer vision, signal processing and machine learning.  Linear dynamical systems (LDSs) are a particularly useful model which builds a compact representation of the spatial and temporal variations in image sequences.  This arises in a range of applications including object recognition, video segmentation, and video synthesis.  We are primarily motivated by video synthesis, since it aligns perfectly with goal of compressive sensing.  Our video compressive sensing model is largely inspired by CS-LDS \cite{sankaranarayanan2010}, which couples compressive sensing with linear dynamical systems to perform video compressive sensing.  Since we are interested in accelerating the image acquisition process of MRI, we extend the CS-LDS model from the spatial domain to the Fourier domain.

\subsection{Notations}

We clarify the notations in this section.  A video can be denoted by a 3D tensor $Y^3 \in \mathbb{R}^{n_x \times n_y \times l}$.  For ease of notation, we vectorize each frame of the video and represent it using 2D matrices; hence, for the rest of the paper, we represent videos as $Y \in \mathbb{R}^{n \times l}$, where $n = n_x n_y$.
We use $X \in \mathbb{R}^{d \times l}$ to denote the state sequence over time, where at each time $t$ the state vector is $\mathbf{x}_t$. We denote the compressive sensed k-t video cube as $Z \in \mathbb{R}^{m \times l}$, where $m$ is the number of measurements. 

We use different notations for row space and column space of matrices.  In particular, for the observation matrix $C \in \mathbb{R}^{n \times d}$, we define the row vector and column vector as follows.  We denote each row of $C$ using a row vector $\mathbf{c}_{(i)} := \mathbf{e}_{(i)} C$, where $\mathbf{e}_{(i)} = (0,\hdots,1\hdots,0)^{\top}$ is a row selector with the $i$th element being 1, $i = 1,2,\hdots,n$.  Similarly, we denote each column of $C$ using a column vector $\mathbf{c}_j := C \mathbf{e}_j$, where $\mathbf{e}_j = (0,\hdots,1,\hdots,0)$ is a column selector with the $j$th element being 1, $j = 1,2,\hdots,d$.

We use $\psi_2$ to denote the 2D wavelet transform.  We define the following operator to denote frame-by-frame wavelet transform for each frame of the video cube $C \in \mathbb{R}^{n \times d}$ as
\begin{equation}
\Psi \mathbf{c}_j = \Psi (\mathbf{c}_j) := \mathrm{vec}(\psi_2 \mathbf{c}^2_j),
\label{eq:waveletopt}
\end{equation}
where $\mathbf{c}^2_j$ denotes matricization of the 3D tensor $\mathbf{c}_j$ by collapsing the first two dimension, resulting in a 2D matrix,
\begin{equation}
\mathbf{c}^2_j = \mathrm{mat}(\mathbf{c}_j),
\end{equation}
for $j = 1,2, \hdots,d$.  

With such a notation, the frame-by-frame wavelet transform for the entire video cube can be represented by
\begin{equation}
\Psi(C) =
\begin{pmatrix}
| & | &  & | \\
\mathrm{vec}(\psi_2 \mathbf{c}^2_1) & \mathrm{vec}(\psi_2 \mathbf{c}^2_2) & \hdots & \mathrm{vec}(\psi_2 \mathbf{c}^2_d) \\
| & | &  & | \\
\end{pmatrix}.
\label{eq:waveletoptvideo}
\end{equation}

%In particular for the observation matrix $C \in \mathbb{R}^{n \times d}$, we use $\mathbf{c}_{j}$ to denote column vectors indexed by $j = 1,2,\cdots,d$, and $\mathbf{c}_{(i)}$ to denote row vectors indexed by $i = 1,2,\cdots,n$.

\subsection{Compressive Sensed Dynamic MRI}

We first introduce the signal model for compressive sensing dynamic MRI.  Dynamic MRI imaging typically takes measurements of a moving object $\mathbf{y}_t \in \mathbb{R}^n$ in the Fourier space, which result in a sequence of Fourier measurements $\mathbf{z}_t \in \mathbb{R}^n$,
\begin{equation}
\mathbf{{z}}_t = \mathcal{F} \mathbf{y}_t + \xi_t,
\end{equation}
where $\xi_t \in \mathbb{R}^n$ is the measurement noise.
Traditional imaging system takes the full Fourier space samples and reconstructs a video using inverse Fourier transform,
\begin{equation}
\mathbf{\hat{y}}_t = \mathcal{F}^{-1} \mathbf{z}_t.
\end{equation}
With the traditional imaging sequence, the sampling rate needs to satisfy the Nyquist sampling theorem, which fundamentally limits the temporal resolution of the dynamic MRI imaging.

Now with compressive sensing, we take partial measurements in the k-t space and increase the temporal resolution of the imaging acquisition process,
\begin{equation}
\mathbf{z}_t = \Phi_t \mathcal{F} \mathbf{y}_t + \xi_t,
\end{equation}
where $\Phi_t \in \mathbb{R}^{m \times n}$ is the measurement matrix, and $\mathbf{z}_t \in \mathbb{R}^{m}$ represents partial Fourier measurements.  Note $\Phi_t$ is a row selector in the Fourier space, and takes the form of a subsampled identity matrix.  Our goal is recover $\mathbf{y}_t$ based on measurement matrix and partial Fourier measurements.

\subsection{Linear Dynamical Systems}

Linear time-invariant dynamical systems (LDS) can be expressed using three components: 1) an observation model that defines the state space and the observation matrix linking observations to state space, 2) state transition model that captures dynamics on the state space, and 3) an initial condition.  

More specifically, the discrete form of LDS can be expressed as
\begin{subequations}
\label{eq:lds}
\begin{align}
\mathbf{y}_{t}  & = C \mathbf{x}_{t} + \mathbf{\omega}_{t} & \hspace{5mm} & \mathbf{\omega}_t \sim \mathcal{N}(0,Q) \\
\mathbf{x}_{t+1} & = A \mathbf{x}_{t} + \mathbf{\nu}_{t} & \hspace{5mm} & \mathbf{\nu}_t \sim \mathcal{N}(0,R)
\end{align}
\end{subequations}
at each time instant ($t = 1,2,\cdots,l$), together with the initial condition $x_0$.  In the above, $\mathbf{y}_t \in \mathbb{R}^{n}$ represents the observation (in our case, the original video frames), $\mathbf{x}_t \in \mathbb{R}^{d}$ are the hidden states, $C \in \mathbb{R}^{n \times d}$ is the observation matrix, $A \in \mathbb{R}^{d \times d}$ is the transition matrix.  $\mathbf{\omega}_t \in \mathbb{R}^{n}$ is the process noise, which include excitation driving the stochastic process and error in the Markov model.  $\mathbf{\nu}_t \in \mathbb{R}^{n}$ is the observation noise, modeling inaccuracies in the LDS model.

We first illustrate the concept of LDS model using a sample video, see Figure~\ref{fig:heartlds}(A).  With such a model, observations $\mathbf{y}_t$ can be represented as linear transformation of the state $\mathbf{x}_t$, corrupted by observation noise, whereas the states $\mathbf{x}_t$ evolve according to a first-order Markov process corrupted by process noise.  The noise terms $\mathbf{\omega}_t$ and $\mathbf{\nu}_t$ are assumed to be temporally white, independent of each other, the states and the observations.  If the noises are assumed to be zero-mean Gaussian spatially, then the LDS model corresponds to a first-order Gaussian Markov random process.  We focus on the Gaussian noise case throughout this paper, with $\mathbf{\omega}_t \sim \mathcal{N}(0,Q)$ and $\mathbf{\nu}_t \sim \mathcal{N}(0,R)$.  

\begin{figure}
\begin{center}
\includegraphics[width = 0.98\textwidth]{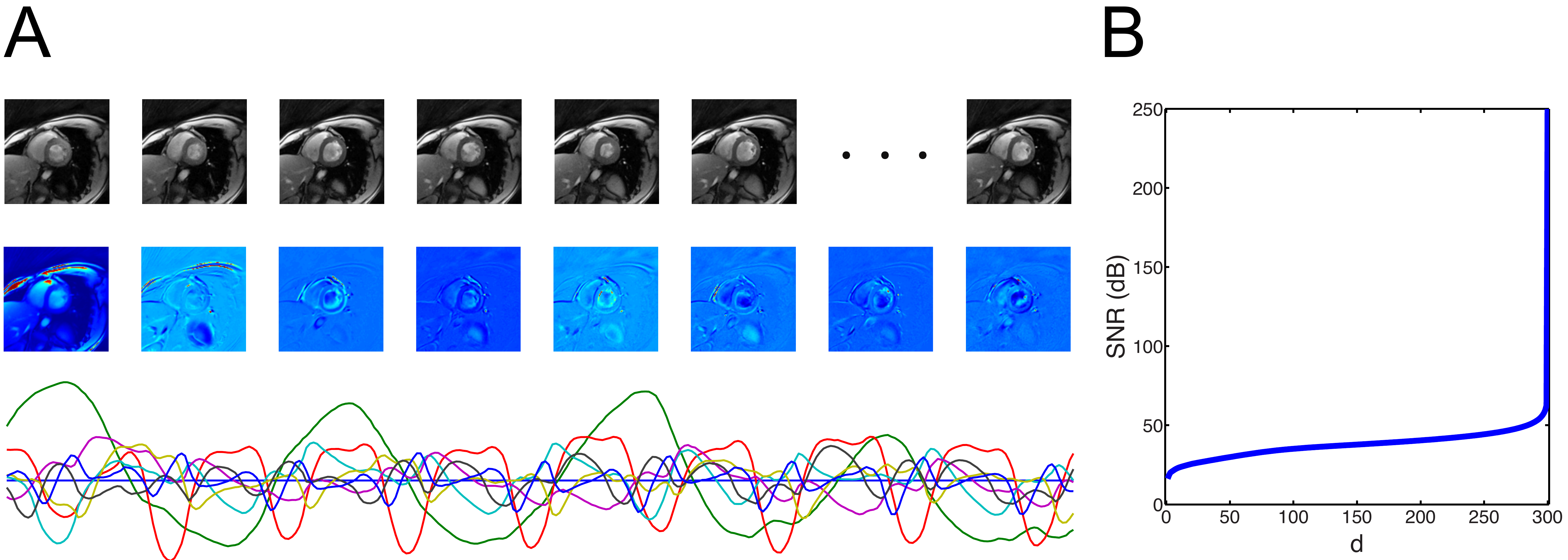}
\end{center}
\caption{(A) Illustration of the LDS model for a sample heart video: (top) sequence of individual frames of a video $Y$, (middle) observation matrix $C$, (bottom) state sequence $X$.  (B) LDS as a good approximation of the original video, where we plot SNR as a function of $d$.}
\label{fig:heartlds}
\end{figure}

 %which is nevertheless a signal-dependent property, motion manifold is a good model to reduce the high-dimensional video cube into a low-dimensional representation. %We note that for periodic or quasi-periodic signals, LDS is a perfect model; even if the assumption does not hold.
%As long as the the signal is quasi-stationary, LDS model is a good representation [NEED REFERENCE].

%We add a compressive measurement model on top of the LDS model, which will be discussed in the next section.

In the case where $d \ll n$, the motion manifold is a good model to reduce the high-dimensional video cube into a low-dimensional representation.  We illustrate such a concept in Figure~\ref{fig:heartlds}(B).  The key promise of using LDS relies on the assumption that high-dimensional signal $\mathbf{y}_t \in \mathbb{R}^n$ can be faithfully represented using low-dimensional state sequence $\mathbf{x}_t \in \mathbb{R}^d$, with $d \ll n$.

Given $Y = (y_1,y_2,\hdots,y_l)$, we can obtain a $d$-dimensional LDS approximation of the original video cube $Y$ through SVD, and measure the accuracy of such an approximation.  When we have an estimate of $Y$, $\hat{Y} = C(d)X(d)$, the reconstruction SNR of $\hat{Y}$ is given as
\begin{equation}
\textrm{SNR} = 10 \log_{10} \frac{\| Y \|^2_{F}}{\| \hat{Y} -Y \|^2_{F}},
\end{equation}
which is a function of $d$.  We obtain reasonably good SNR even at low $d$, as shown in Figure~\ref{fig:heartlds}(B).  %For the heart video we have illustrated in Figure~\ref{fig:heartlds}, $d = 4$ is sufficient to approximate the original video.

\subsection{Compressive Measurement Model}

Now with compressive measurements in the Fourier space, we have
\begin{subequations}
\begin{align}
\mathbf{z}_t &= \Phi_t \mathcal{F} C \mathbf{x}_{t} + \mathbf{\omega}_{t} & \hspace{5mm} & \mathbf{\omega}_t \sim \mathcal{N}(0,Q) \\
\mathbf{x}_{t+1} &= A \mathbf{x}_{t} + \mathbf{\nu}_{t} & \hspace{5mm} & \mathbf{\nu}_t \sim \mathcal{N}(0,R),
\end{align}
\end{subequations}
where $\mathcal{F}$ is the multi-dimensional Fourier transform in the image domain $\Omega$.  $\Phi_t \in \mathbb{R}^{m \times n}$ is the measurement matrix, essentially the row selector which stipulates where we sample in the $k$ space during the image acquisition process,
\begin{equation*}
\Phi_t = \begin{pmatrix}
1 & 0 & 0 & \cdots & 0 \\
0 & 0 & 1 & \cdots & 0 \\
\vdots & \vdots & \vdots &\ddots & \vdots \\
0 & 0 & 0 & \cdots & 1 \\
\end{pmatrix}.
\end{equation*}

Our measurement model is comprised of two components, a time-invariant component and a time-variant component, as proposed in \cite{sankaranarayanan2010, sankaranarayanan2013},
\begin{equation}
\mathbf{z}_t =
\begin{pmatrix}
\bar{\mathbf{z}}_t \\
\tilde{\mathbf{z}}_t \\
\end{pmatrix} =
\begin{pmatrix}
\bar{\Phi} \mathcal{F} \\
\tilde{\Phi}_t \mathcal{F} \\
\end{pmatrix} \mathbf{y}_t.
\end{equation}
Note that $\bar{\Phi}$ is the time-invariant component, which is designed to facilitate estimation for the state sequence.  $\tilde{\Phi}_t$ is the time-variant component to allow innovation during the acquisition process, due to the ephemeral nature of videos.

The overall architecture of kt-CSLDS is shown in Figure~\ref{fig:manifoldlds}, where we take compressive measurements of the original video signal.  Each frame of the video is a high-dimensional signal is projected to a low-dimensional representation that follows a Markov process. 

\begin{figure}
\begin{center}
\includegraphics[width = 0.98\textwidth]{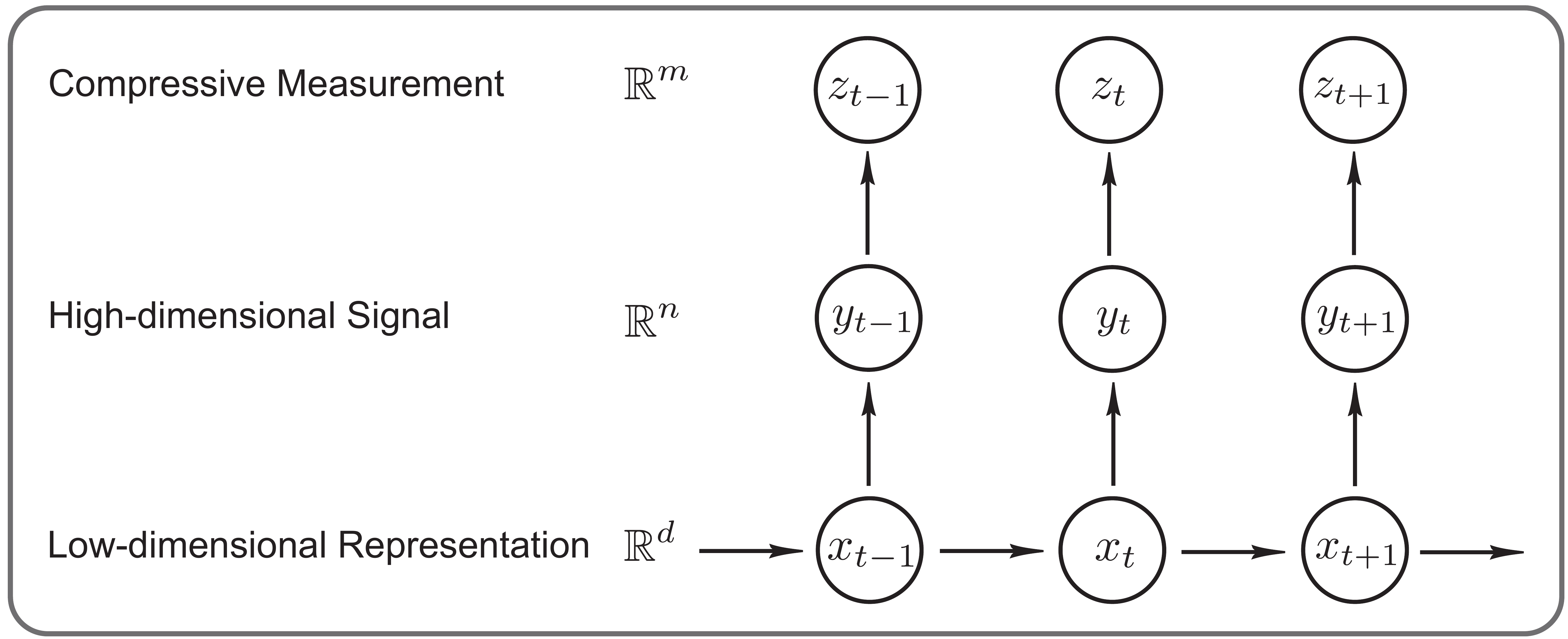}
\end{center}
\caption{Architecture of the kt-CSLDS model. High-dimensional signal is projected to a low-dimensional representation that follows a Markov process.  The compressive measurements are observations drawn from the high-dimensional signal.}
\label{fig:manifoldlds}
\end{figure}

The design of the measurement model is critical to the success of video compressive sensing.  The time-invariant component needs to satisfy the observability condition for LDS.  In the Fourier domain, the low-frequency content of a video changes marginally, while the high-frequency content changes more drastically from frame to frame.  We therefore sample the low-frequency content more densely.  Note in both kt-FOCUSS and MASTeR, low-frequency domain is sampled in a similar fashion.  The time-variant component is designed to satisfy the incoherence assumptions for compressive sensing, promoting randomness of the Fourier samples in the temporal domain.  Within each frame, we sample the Fourier space according to the theoretical results obtained by \cite{krahmer2012}.

\section{Compressive Sensing Reconstruction Algorithm}

\subsection{Algorithm Flow}

The overall architecture of imaging acquisition and video reconstruction process is summarized in Figure~\ref{fig:heartarchitect}.  Upon acquiring Fourier samples using the aforementioned sampling strategy, we reconstruct the video frames for dynamic MRI via a two-step procedure.

We first estimate the state sequence $\hat{X}$ based on the time-invariant samples using system identification.  We then reconstruct the observation matrix $\hat{C}$ based on both time-invariant and time-variant measurements by exploring certain sparsity assumptions.  The final video reconstruction can be obtained by
\begin{equation}
\hat{Y} = \hat{C} \hat{X}.
\end{equation}

\begin{figure}
\begin{center}
\includegraphics[width = 0.98\textwidth]{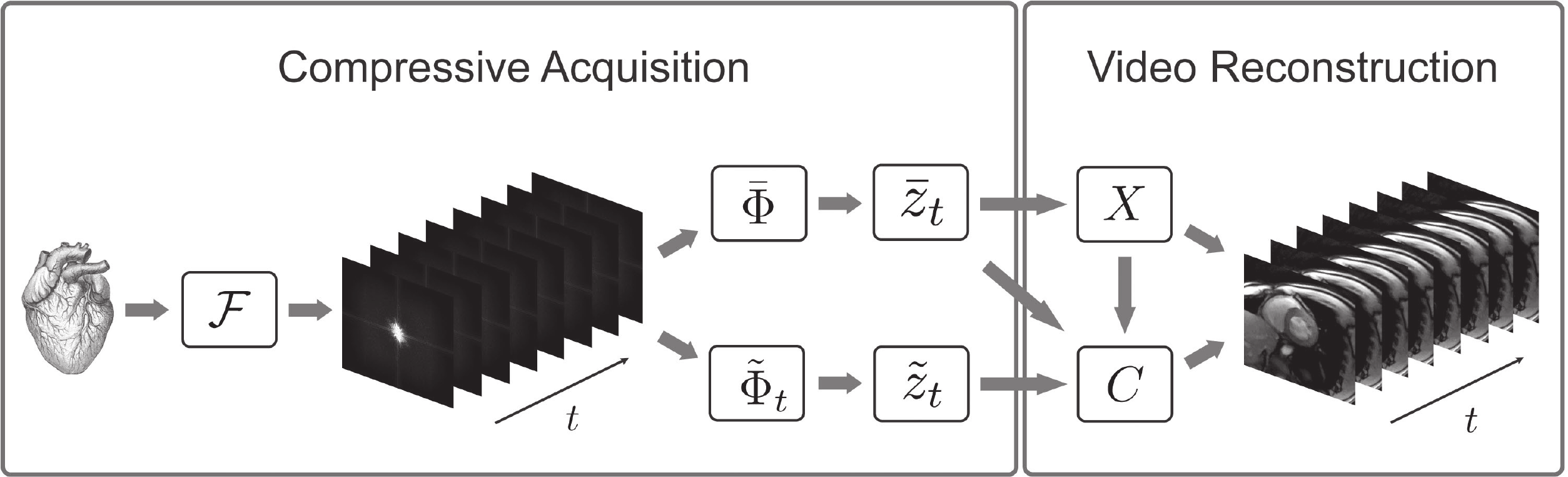}
\end{center}
\caption{Architecture of the kt-CSLDS imaging and reconstruction processes.  During the imaging acquisition process of dynamic MRI, compressive samples are obtained in the Fourier space.  We then perform video reconstruction based on the compressive measurements.}
\label{fig:heartarchitect}
\end{figure}

Here we set up the general framework for algorithm flow, whose details will be discussed in the following two sections.  In Section~\ref{stateseq} we discuss two different methods for state sequence estimation.  In Section~\ref{obsrecon}, we derive an efficient numerical algorithm for observation matrix reconstruction and prove its theoretical convergence.

\subsection{State Sequence Estimation}
\label{stateseq}

State sequence $\mathbf{x}_t$ can be estimated based on the time-invariant component,
\begin{subequations}
\label{eq:ldsti}
\begin{align}
%\bar{\mathbf{z}}_t &= \bar{\Phi} \mathcal{F} C \mathbf{x}_t + \mathbf{\omega}_t & \hspace{1mm} & \mathbf{\omega}_t \sim \mathcal{N}(0,Q) \\
%\mathbf{x}_{t+1} &= A \mathbf{x}_t + \mathbf{\nu}_t & \hspace{1mm} & \mathbf{\nu}_t \sim \mathcal{N}(0,R)
\bar{\mathbf{z}}_t &= \bar{\Phi} \mathcal{F} C \mathbf{x}_t + \mathbf{\omega}_t & \mathbf{\omega}_t \sim \mathcal{N}(0,Q) \\
\mathbf{x}_{t+1} &= A \mathbf{x}_t + \mathbf{\nu}_t & \mathbf{\nu}_t \sim \mathcal{N}(0,R).
\end{align}
\end{subequations}
Since that Fourier transform is a linear operator, this sub-system is a linear time-invariant system.  Given that there is no input data, the system is driven purely by stochastic noise.

%\begin{definition}
Given compressive measurements of video data $\bar{Z} = (\bar{z}_1, \bar{z}_2, \hdots, \bar{z}_t, \hdots, \bar{z}_{l})$, we assume they are generated by the time-invariant LDS described in Eq.~\eqref{eq:ldsti}.  Our goal is to determine the order of the unknown system $d$, and the forward Kalman filter state sequence $\hat{X} = (x_1, x_2, \hdots, x_t, \hdots, x_{l})$ up to a similarity transformation.
%\end{definition}

There exist many ways to perform state inference and system identification for the LDS model.  State inference refers to the process of estimating hidden states over time $\hat{X} := (x_1,\hdots,x_t,\hdots,x_{l})$ given the observations $\bar{Z} := (\bar{z}_1,\hdots,\bar{z}_t,\hdots,\bar{z}_l)$ and parameters $\theta := \{A,C,Q,R\}$.  On the other hand, system identification involves finding the parameters $\theta$ and the distribution over hidden states $p(X|\bar{Z},\theta)$ that maximizes the likelihood of the observed data $\bar{Z}$.

When one uses the recursive formulation as stated in \eqref{eq:ldsti}, it has a connection to the Kalman filter.  The stochastic LDS models the distribution of outputs $p(\bar{z}_{1:l})$, and the inference problem for LDS aims to estimate the distribution over hidden states $p(x_t | \bar{z}_{1:l})$.  The inference can be carried out recursively, by combining a forward pass and a backward pass.  The forward pass takes the initial state $x_0$ together with a collection of observation states $y_{1:t}$, and computes $x_t$ recursively, resulting in the Kalman filter.  The backward pass takes the observations from $\bar{z}_{l}$ to $\bar{z}_{t+1}$, and corrects the results from the forward pass by evaluating the influence of future observations, which is also known as the Rauch-Tung-Striebel (RTS) equation.

\subsubsection{Review of System Identification}

Generally speaking, there are two types of methods for system identification.  One approach is to obtain parameters $\theta$ and distribution $p(X|\bar{Z},\theta)$ using maximum likelihood solution through iterative techniques such as expectation maximization (EM).  The EM approach utilizes the Kalman filter and Kalman smoother, which requires the entire observation sequences.  EM guarantees convergence to a local maximum in the likelihood surface, and is sensitive to initial condition.  Another approach for system identification is to use subspace methods to obtain solutions, which is known as subspace identification.

In subspace identification, it is desirable to find the minimal model order for the state space representation under the constraint that the reduced model approximates the output data.  It is well-known that the minimal order is equal to the rank of the block Hankel matrix, defined as
\begin{equation}
\label{eq:hankel}
\mathcal{H}_{d}(z) =
\begin{pmatrix}
z_{1} & z_{2} & \cdots & z_{l-d+1} \\
z_{2} & z_{3} & \cdots & z_{l-d+2} \\
\vdots & \vdots & \ddots & \vdots \\
z_{d} & z_{d+1} & \cdots & z_{l}
\end{pmatrix}.
\end{equation}
Therefore, subspace identification methods typically exploit the rank of block Hankel matrix, and relies on matrix decomposition to obtain state sequence estimates $\hat{X}$ as well as the realization of the state space model parameterized by $\{A,C,Q,R\}$.  Algorithm details vary among different subspace identification methods: a) construction of the block Hankel matrices differ depending on whether it is covariance driven or data driven, b) the matrix decomposition methods vary using different user defined weighting matrices in the projection methods.  In several classical subspace identification algorithms, oblique projection is employed when there exist both output data and input data, which reduces to orthogonal projection in the case of stochastic identification.  These algorithms include principle component analysis (PCA), unweighted principle component analysis (UPCA) \cite{arun1990}, canonical variate analysis (CVA) \cite{larimore1990}, numerical algorithms for subspace identification (N4SID) \cite{overschee1994}, and multivariate output error state space (MOESP) \cite{verhaegen1994}.  These methods can determine the order of the system for the state space model, and estimate system matrices $\{A,C\}$ up to a similarity transformation.  Note that subspace identification methods mentioned above use orthogonal projections and can be computationally expensive.  It is noteworthy that N4SID provides asymptotically optimal solution for the forward Kalman filter state sequence $\hat{X}$, in the sense of maximum likelihood.  However, the memory storage and computation requirement of N4SID and other subspace methods are prohibitively expensive for video data.  As pointed out in \cite{doretto2003}, under mild conditions, one can obtain a closed-form solution for a canonical model realization.  We adopt such a strategy in this paper in favor of its algebraic simplicity and computational efficiency.

\subsubsection{Canonical Model Realization}

In stochastic identification, the goal is typically to determine the system matrices $\theta = \{A,C,Q,R\}$ up to a similarity transformation.  Obtaining $\{A,C,Q,R\}$ is also called a \emph{realization} of the system.  The ambiguity of the system identification is well-known, in a sense that there is no unique choice of system matrices which can generate the same sample path given suitable initial condition.  As long as $T \in \mathbb{R}^{d \times d}$ is invertible, one can generate the same dynamics by substituting $A$ with $TAT^{-1}$, $C$ with $CT^{-1}$, $Q$ with $TQT^{-1}$, and initial condition $x_0$ with $Tx_0$.  Given the bilinear product between $C$ and $x_t$, it immediately follows that any estimate of the forward Kalman filter state sequence $\hat{X}$ is accurate only up to a similarity transformation.

In order to obtain a unique realization for the LDS, one chooses a representative from these equivalent solutions, which results in the so-called \emph{canonical model realization} \cite{doretto2003}.  This can be achieved by imposing additional constraints or imposing regularization on the solution.

\subsubsection{Canonical Model Realization with Truncated SVD}

\begin{algorithm}[t]
\caption{State Sequence Estimate with SVD}
\label{alg:frx}
\begin{algorithmic}
\STATE{\textbf{Input:}} Fourier measurements $\bar{Z}$ acquired by time-invariant measurement matrix $\bar{\Phi}$ and system order $d$
\STATE \textbf{1}$^{\circ}$ Formulate block Hankel matrix $\mathcal{H}_d (\bar{z})$
\STATE \textbf{2}$^{\circ}$ Perform SVD $\mathcal{H}_{d}(z) = U \Sigma V^{\top}$
\STATE \textbf{3}$^{\circ}$ Keep the first $d$ eigensystems $U_d$, $\Sigma_d$, $V_d$
\STATE \textbf{4}$^{\circ}$ Calculate state sequence estimate $\hat{X} = \Sigma_d V_d^{\top}$
\end{algorithmic}
\end{algorithm}

\begin{corollary}[Canonical model realization, \cite{doretto2003}]
\label{co:cmr}
Suppose one has access to the full video data in the spatial domain $Y = (y_1, y_2, \hdots, y_t, \hdots, y_l)$.  Assuming the observation matrix $C \in \mathbb{R}^{n \times d}$ of the canonical model has orthonormal columns, i.e., $C^{\top}C = I$, then one can obtain a closed-form solution for the forward Kalman filter state sequence $\hat{X}$ and canonical model realization $\hat{C}$ based on the simplest form of Hankel matrix $\mathcal{H}_{1,l}$.
\end{corollary}
\begin{proof}
Once formulating the simplest form of block Hankel matrix $\mathcal{H}_{1}(y)$, we note the following relationship with state sequence $X = (x_1,x_2,\hdots,x_t,\hdots,x_l)$ and noise term $W = (\omega_1,\omega_2,\hdots,\omega_t,\hdots,\omega_l)$,
\begin{equation}
\mathcal{H}_{1}(y) = C X + W. \nonumber
\end{equation}
The estimation for observation matrix $C$ and state sequence $X$ can be formulated as
\begin{equation}
\hat{C}, \hat{X} = \arg \min_{C,X} \| \mathcal{H}_{1,l}(y) - C X \|_{F}^{2}. \nonumber
\end{equation}
Given the bilinear product between $C$ and $X$, one can immediately inspect that the solution is not unique.  Since we have imposed additional constraint $C^{\top}C = I$, one can obtain a canonical model realization.  It follows from the fixed rank property of SVD \cite{golub1989} that a unique closed-form solution can be obtained through the SVD:
\begin{equation}
\mathcal{H}_{1}(y) = U \Sigma V^{\top} \hspace{4mm} \text{where} \hspace{4mm} U^{\top}U = I \hspace{4mm}V^{\top}V = I. \nonumber
\end{equation}
Note the system order $d$ can be determined from the rank of the block Hankel matrix, which leads to the following
\begin{equation}
\hat{C} = U \hspace{4mm} \hat{X} = \Sigma V^{\top},
\end{equation}
where both estimates are closed-form solutions.
\end{proof}

The above result sheds light on a simpler path of estimating state sequence $\hat{X}$, without the computational burden of system identification methods.  We remark that canonical model realization based on the simplest form of Hankel matrix \cite{doretto2003} essentially does not exploit the structure of LDS.  Once we consider a higher degree Hankel matrix, the structure and observability of LDS comes into consideration.

Formally, an LDS is said to be observable if, for any possible state sequence, the current state can be determined in finite time using only the outputs.  Less formally, observability refers to the idea that it is possible to determine the behavior of the entire system based on merely the system's outputs.  Conversely, an LDS is said to be unobservable if the current values of some states cannot be determined through output sensors.  There exists a convenient test for observability.

\begin{lemma}[Observability, \cite{brockett1970}]
For an LDS, equipped with system matrices $(C,A)$ and state space dimension $d$, the system is observable if the observability matrix
\begin{equation}
\mathcal{O}(C,A) =
\begin{pmatrix}
C \\
CA \\
CA^2 \\
\vdots \\
CA^{d-1} \\
\end{pmatrix}
\end{equation}
is full rank.
\end{lemma}
The rationale for this test is that if $\mathcal{O}(C,A)$ is rank $d$, then each of the $d$ states is viewable through linear combinations of the system output $Y = (y_1,y_2,\hdots,y_t,\hdots,y_l)$.

\begin{theorem}
\label{th:cmr}
Given time-invariant compressive measurements $\bar{\Phi}$ and Fourier video data $\bar{Z} = (\bar{z}_1,\bar{z}_2,\hdots,\bar{z}_t,\hdots,\bar{z}_l)$, suppose the observability matrix $\mathcal{O}(\bar{\Phi} \mathcal{F} C, A)$ is full rank, then there exists a closed-form solution for the forward filter state sequence $\hat{X}$.
\end{theorem}
\begin{proof}
Given the video compressive sensing model, we use the time-invariant component to estimate the state sequence.  We formulate the block Hankel matrix based on the Fourier measurements $\mathcal{H}_{d}$, and by denoting $\bar{C} = \bar{\Phi} \mathcal{F} C$,
\begin{equation}
%\begin{align}
\begin{split}
\mathcal{H}_d(\bar{z}) & =:
\begin{pmatrix}
\bar{z}_1 & \bar{z}_2 & \hdots & \bar{z}_{l-d+1} \\
\bar{z}_2 & \bar{z}_3 & \hdots & \bar{z}_{l-d+2} \\
\vdots & \vdots & \ddots & \vdots \\
\bar{z}_d & \bar{z}_{d+1} & \hdots & \bar{z}_l \\
\end{pmatrix} \nonumber \\ 
& =
\begin{pmatrix}
\bar{C} x_1 & \bar{C} x_2 & \hdots & \bar{C} x_{l-d+1} \\
\bar{C}A x_1 & \bar{C}A x_2 & \hdots & \bar{C}A x_{l-d+1} \\
\vdots & \vdots & \ddots & \vdots \\
\bar{C}A^{d-1} x_1 & \bar{C}A^{d-1} x_{2} & \hdots & \bar{C}A^{d-1} x_{l-d+1} \\
\end{pmatrix} \nonumber \\
& =
\begin{pmatrix}
\bar{C} \\
\bar{C}A \\
\vdots \\
\bar{C}A^{d-1} \\
\end{pmatrix}
\begin{pmatrix}
x_1 & x_2 & \hdots & x_{l-d+1} \\
\end{pmatrix} \nonumber \\
& =
\mathcal{O}(\bar{\Phi} \mathcal{F} C, A)
\begin{pmatrix}
x_1 & x_2 & \hdots & x_{l-d+1} \\
\end{pmatrix}.
\end{split}
%\end{align} 
\end{equation}
Under the assumption that $\mathcal{O}(\bar{\Phi}\mathcal{F}C,A)$ is full-rank, the LDS is observable, according to Lemma 3.4.  Therefore, one can obtain a canonical model realization through the SVD
\begin{equation}
\mathcal{H}_d(\bar{z}) = \mathcal{O}(\bar{\Phi}\mathcal{F}C,A)
\begin{pmatrix}
x_1 & x_2 & \hdots & x_{l-d+1}
\end{pmatrix}
= \tilde{U} \tilde{\Sigma} \tilde{V}^{\top}.
\end{equation}
This leads to the estimate for the state sequence:
\begin{equation}
\hat{X} = \tilde{\Sigma} \tilde{V}^{\top},
\end{equation}
as a modified closed-form solution.
\end{proof}

The above result leads to Algorithm~\ref{alg:frx}, where one can estimate the state sequence based on a very simple procedure using SVD.

\subsubsection{Canonical Model Realization without SVD}

Algorithm~\ref{alg:frx} exploits the full rank of the block Hankel matrix, which represents the model complexity.  It is well-known the rank of block Hankel matrix can be corrupted when there is noise in the data.  Moreover, it is often desirable in system identification to reduce the model complexity.  In the view of video compressive sensing, it is favorable to obtain a most compact representation of the video data and perform computation on the corresponding low-dimensional manifold.  We thus extend Algorithm~\ref{alg:frx} to the low rank case.

We formulate the system identification problem as follows:
\begin{equation}
\label{eq:lr}
\hat{C}, \hat{X} = \arg \min_{C,X} \| C X - M \|_{F}^2 \hspace{3mm} \text{s.t.} \hspace{3mm} M = \mathcal{H}_{1,1,l}(z) \hspace{3mm} \text{rank}(CX) = d
\end{equation}
where a low-rank factorization is sought, resulting in spatial factor $C \in \mathbb{R}^{n \times d}$ and temporal factor $X \in \mathbb{R}^{d \times l}$.  Such an approach is designed to alleviate possible corruption of noise, which can increase the rank of block Hankel matrix.  We adopt a low-rank factorization algorithm based on nonlinear successive over-relaxation (SOR) \cite{wen2010}.  This results in Algorithm~\ref{alg:lrx}, which avoids the computation burden of SVD and obtains estimation for both the observation matrix $\hat{C}$ and state sequence $\hat{X}$.

\begin{algorithm}[t]
\caption{State Sequence Estimate without SVD}
\label{alg:lrx}
\begin{algorithmic}
\STATE \textbf{Input:} video data $\bar{Z}$, system order $d$
\STATE \textbf{1}$^{\circ}$ Formulate simplest block Hankel matrix $\mathcal{H}_{1,1,l}(z) = [\bar{z}_1, \bar{z}_2, \hdots, \bar{z}_t, \hdots, \bar{z}_{l}]$
\STATE \textbf{2}$^{\circ}$ Initialize $k=0$, $C^0 \in \mathbb{R}^{n \times d}$, $X^0 \in \mathbb{R}^{d \times l}$, $\omega = 1$, $\tilde{\omega} > 1$, $\delta > 0$, $\gamma_1 \in (0,1)$
\WHILE{not converged}
\STATE \textbf{3}$^{\circ}$ Set $(C, X, M) = (C^k, X^k, M^k)$
\STATE \textbf{4}$^{\circ}$ Compute $M_{\omega} \leftarrow \omega M^k + (1-\omega) C^k X^k$
\STATE \textbf{5}$^{\circ}$ Compute $C_{+}(\omega) \leftarrow M_{\omega} X^{\top} (X X^{\top})^{\dag}$
\STATE \textbf{6}$^{\circ}$ Compute $X_{+}(\omega) \leftarrow (C_{+}(\omega)^{\top} C_{+}(\omega))^{\dag}( C_{+}(\omega)^{\top} M_{\omega})$
\STATE \textbf{7}$^{\circ}$ Compute $M_{+}(\omega) \leftarrow \mathcal{H}_{1,1,l}(z)$
\STATE \textbf{8}$^{\circ}$ Compute residual ratio $\gamma(\omega) \leftarrow {\|M-C_{+}(\omega) X_{+}(\omega)\|_F}/{\|M-CX\|_F}$
\IF{$\gamma(\omega) \ge 1$}
\STATE Set $\omega = 1$ and go to \textbf{4}$^{\circ}$
\ENDIF
\STATE \textbf{9}$^{\circ}$ Update $(C^{k+1}, X^{k+1}, M^{k+1}) = (C_{+}(\omega), X_{+}(\omega), M_{+}(\omega))$
\STATE Update $k \leftarrow k+1$
\IF{$\gamma(\omega) \ge \gamma_1$}
\STATE Set $\delta = \max(\delta, 0.25(\omega-1))$ and $\omega = \min(\omega+\delta,\tilde{\omega})$
\ENDIF
\ENDWHILE
\end{algorithmic}
\end{algorithm}

\begin{corollary}
\label{co:lr}
There exists at least a subsequence $\{(C^k,X^k,M^k)\}$ generated by Algorithm~\ref{alg:lrx} that satisfies the first-order optimality conditions of \eqref{eq:lr} in the limit.
\end{corollary}
\begin{proof}
We omit the proof here, since it is an immediate result of Theorem 3.5 in \cite{wen2010}.
\end{proof}

The estimated state sequence by factorizing the Hankel matrix is only accurate up to a linear transformation.  In other words, there is no unique solution for the state sequence estimation.  Such an ambiguity poses difficulty for reconstruction of the observation matrix.  This difficulty is resolved by a joint sparsity assumption, as we will describe in the next session.

\subsection{Observation Matrix Reconstruction}
\label{obsrecon}

Once we have estimated the state sequence $\hat{\mathbf{x}}_t$, the next component of the algorithm is to reconstruct the observation matrix $C$.

\subsubsection{Joint Structured Sparsity}

Denoting observation matrix as $C \in \mathbb{R}^{n \times d}$, we formulate the reconstruction model as follows:
%Let $\mathbf{c}_1, \mathbf{c}_2, \cdots, \mathbf{c}_d$ be the columns of $C$, and $\mathbf{c}_{(1)}, \mathbf{c}_{(2)}, \cdots, \mathbf{c}_{(n)}$ be the rows of $C$. We formulate the reconstruction model as follows:
\begin{equation}
\label{eq:jss}
\min_{C} \hspace{2mm} \alpha \sum_{i=1}^{n} \| \mathbf{e}_{(i)} \Psi(C) \|_2 + \beta \sum_{j=1}^{d} \| \Psi(C) \mathbf{e}_j \|_1 
+ \sum_{t=1}^{l} \frac{1}{2} \| \mathbf{z}_t - \Phi_t \mathcal{F} (C \hat{\mathbf{x}}_t) \|_2^2,
\end{equation}
where $\Psi$ denotes the frame-by-frame wavelet transform operator, defined in \eqref{eq:waveletoptvideo}.  The first two regularization terms concern structured sparsity for the observation matrix.  The first term is the joint sparsity regularization, $\sum_{i=1}^{n} \| \mathbf{e}_{(i)} \Psi(C) \|_2$, which encourages all the columns $\Psi(C)$ to share a common yet small support.  The second term, $\sum_{j=1}^{d} \| \Psi(C) \mathbf{e}_j \|_1$, demands $ \Psi(C) \mathbf{e}_j$ to be sparse under wavelet transform, based on the assumption that each frame of the observation matrix is image-like.  The joint sparsity is critical to the success of reconstruction, due to the ambiguity introduced by the non-uniqueness within the state sequence estimation.  

Computationally one immediately notices the first two terms are non-smooth and both involve $C$.  Moreover, the first term operates on the row space of matrix $C$ while the second term operates on the column space of matrix $C$.  In addition, the amount of data and the number of variables are large in our application. For these reasons, it is difficult to solve the optimization problem by off-the-shell algorithms for $\ell_1$ minimization.  We propose to apply the alternating direction method of multipliers (ADMM) \cite{gabay1876} in such a way that all subproblems are easy to solve and can handle a large amount of data in a short time.

\subsubsection{Alternating Direction Method of Multipliers}

ADMM combines variable splitting techniques with the augmented Lagrangian method for solving constrained optimization with separable objective functions. It is also referred to as the alternating direction augmented Lagrangian method by several groups in the community.\  Alternating direction methods originated from solving PDEs \cite{douglas1956, peaceman1955} and were later extended to solving variational problems associated with PDEs \cite{gabay1876}. Recently, there has been a surge of interest in fast optimization algorithms using ADMM methodology for solving $\ell_1$ and TV regularized problems \cite{goldstein2008, afonso2010, goldfarb2010, yang2011}. A state-of-art algorithm solved for group sparsity problems, which include joint sparsity \cite{deng2011}. Our new algorithm is based on ADMM and is optimized for the best computational efficiency.

We first introduce additional variables to split the energy between different regularization terms, which results in the following constrained optimization,
\begin{equation}
\label{prob17}
\begin{split}
\min_{U, V, C}  & \hspace{2mm}  \alpha \sum_{i=1}^{n} \| \mathbf{u}_{(i)} \|_2 + \beta \sum_{j=1}^{d} \| \mathbf{v}_{j} \|_1
+ \sum_{t=1}^{l} \frac{1}{2} \| \mathbf{z}_t - \Phi_t \mathcal{F} (C \hat{\mathbf{x}}_t) \|_2^2 \\
\text{s.t.\hspace{2mm}}  &  \hspace{2mm} \mathbf{u}_{(i)} = \mathbf{e}_{(i)} \Psi(C) \hspace{6mm} \mathbf{v}_{j} = \Psi(C) \mathbf{e}_{j},
\end{split}
\end{equation}
where we have introduced variables $U$ and $V$.  The rows of $U \in \mathbb{R}^{n \times d}$ are $\mathbf{u}_{(i)}$, $i = 1,2,\cdots,n$, and the columns of $V \in \mathbb{R}^{n \times d}$ are $\mathbf{v}_{j}$, $j = 1,2,\cdots,d$.

To apply ADMM, we introduce the augmented Lagrangian of problem \eqref{prob17},
\begin{equation}
\label{prob18} 
\begin{split}
\min_{U, V, C} \hspace{2mm} & \alpha \sum_{i=1}^{n} \| \mathbf{u}_{(i)} \|_2 + \beta \sum_{j=1}^{d} \| \mathbf{v}_{j} \|_1 
+ \sum_{t=1}^{l} \frac{1}{2} \| \mathbf{z}_t - \Phi_t \mathcal{F} (C \hat{\mathbf{x}}_t) \|_2^2 \\
& + \frac{\alpha \mu}{2} \sum_{i=1}^{n} \| \mathbf{u}_{(i)} - \mathbf{e}_{(i)} \Psi(C) - \mathbf{\kappa}_{(i)} \|_2^2
+ \frac{\beta \mu}{2} \sum_{j=1}^{d} \| \mathbf{v}_{j} - \Psi(C) \mathbf{e}_{j} - \mathbf{\lambda}_{j} \|_2^2,
\end{split}
\end{equation}
where we use $\mathcal{L}(U,V,C)$ to denote the augmented Lagrangian, and use $\mathbf{\kappa}_{(i)}$ and $\mathbf{\lambda}_{j}$ to denote the scaled Lagrange multipliers.  When $C$ is fixed, minimizing over $U$ and $V$ are independent. Therefore, we apply ADMM to \eqref{prob18} in which we alternate between minimizing its objective function over $U,V$ with fixed $C$ and minimizing it over $C$ with fixed $U,V$, along with the updates to $\mathbf{\lambda}_{j}$, $\mathbf{\kappa}_{(i)}$. We summarize the algorithm flow in Algorithm \ref{alg3} and explain in the next few subsections how to efficiently solve each of its subproblems.

\begin{algorithm}[t]
\caption{Joint Structured Sparsity Optimization}
\label{alg3}
\begin{algorithmic}
\STATE Initialize $U \in \mathbb{R}^{n \times d}$, $V \in \mathbb{R}^{n \times d}$, $C \in \mathbb{R}^{n \times d}$, $k = 0$.
\WHILE {convergence criteria not met}

\STATE \textbf{$U$-subproblem:} \\
\STATE  $ U^{k+1} = \arg \underset{U} {\text{min}} \hspace{2mm} \sum_{j=1}^{n} \alpha \| \mathbf{u}_{(i)} \|_2 + \frac{\alpha \mu}{2} \| \mathbf{u}_{(i)} - \mathbf{e}_{(i)} \Psi(C^k) - \mathbf{\kappa}_{(i)}^k \|_2^2$

\STATE \textbf{$V$-subproblem:} \\
\STATE $ V^{k+1} = \arg \underset{V} {\text{min}} \hspace{2mm} \sum_{j=1}^{d} \beta \| \mathbf{v}_{j} \|_1 + \frac{\beta \mu}{2} \| \mathbf{v}_{j} - \Psi(C^k) \mathbf{e}_{j} - \mathbf{\lambda}_{j}^k \|_2^2$

\STATE \textbf{$C$-subproblem:} \\
\STATE $ C^{k+1} = \arg \underset{C} {\text{min}} \hspace{2mm}
\frac{\alpha \mu}{2} \sum_{i=1}^{n} \| \mathbf{u}_{(i)}^{k+1} - \mathbf{e}_{(i)} \Psi(C) - \mathbf{\kappa}_{(i)}^k \|_2^2
+ \frac{\beta \mu}{2} \sum_{j=1}^{d} \| \mathbf{v}_{j}^{k+1} - \Psi(C) \mathbf{e}_{j} - \mathbf{\lambda}_{j}^k \|_2^2 
+ \sum_{t=1}^{l} \frac{1}{2} \| \mathbf{z}_t - \Phi_t \mathcal{F} (C \hat{\mathbf{x}}_t) \|_2^2$

\STATE \textbf{Multipliers update:} \\
\STATE $ \mathbf{\kappa}_{(i)}^{k+1} = \mathbf{\kappa}_{(i)}^{k} - \gamma (\mathbf{u}_{(i)}^{k+1} - \Psi \mathbf{c}_{(i)}^{k+1} )$
\STATE $ \mathbf{\lambda}_{j}^{k+1} = \mathbf{\lambda}_{j}^{k} - \gamma  (\mathbf{v}_{j}^{k+1} - \Psi \mathbf{c}_{j}^{k+1} )$ 

\STATE $k = k+1$
\ENDWHILE
\end{algorithmic}
\end{algorithm}

%\item{Joint Sparsity}
\subsubsection{Joint Sparsity}

The $U$-subproblem models the joint sparsity between the different spatial factors within the observation matrix.  Noticing the optimization is independent with respect to each row $\mathbf{u}_{(i)}$ of $U$, we therefore solve for  $i = 1, 2, \hdots, n$,
\begin{equation}
\label{eq:uprob}
\mathbf{u}_{(i)}^{k+1} = \arg \min_{\mathbf{u}_{(i)}} \hspace{1mm} \alpha \| \mathbf{u}_{(i)} \|_2 + \frac{\alpha \mu}{2} \| \mathbf{u}_{(i)} - \mathbf{e}_{(i)} \Psi(C^k) - \mathbf{\kappa}_{(i)}^k \|_2^2.
\end{equation}

\begin{lemma}[Shrinkage for $\ell_2$ norm, \cite{deng2011}]
For any $\lambda, \mu >0$ and $x, y \in \mathbb{R}^n$, the minimizer to
\begin{equation*}
\min_{y} \lambda \|y\|_2 + \frac{\mu}{2} \|y-x\|_2^2
\end{equation*}
is given by
\begin{equation}
y = \mathcal{S}_2 \big(x,\frac{\lambda}{\mu} \big) := \max \big\{\|x\|_2-\frac{\lambda}{\mu},0 \big\} \odot \frac{x}{\|x\|_2},
\end{equation}
where $\odot$ denotes  component-wise product and $\mathcal{S}_2$ stands for $\ell_2$-shrinkage.
\end{lemma}

One can derive closed-form solution to the U-subproblem \eqref{eq:uprob},
\begin{equation}
\mathbf{u}^{k+1}_{(i)} = \mathcal{S}_2 \big( \mathbf{e}_{(i)} \Psi(C^k) + \mathbf{\kappa}^k_{(i)}, \frac{1}{\mu} \big),\quad i = 1, 2, \hdots, n.
\end{equation}

%\begin{itemize}
%\item{Wavelet Sparsity}
\subsubsection{Wavelet Sparsity}

The $V$-subproblem concerns the wavelet sparsity, and is reduced to a sequence of the same $\ell_1$ minimization problems with different data.  Since computation on each column of $V$ matrix $\mathbf{v}_{j}$ is completely decoupled, we solve for each column independently.  For $j = 1, 2, \hdots, d$,
\begin{equation}
\label{eq:vprob}
\mathbf{v}_{j}^{k+1} = \arg \min_{ \mathbf{v}_{j} } \hspace{1mm} \beta \| \mathbf{v}_{j} \|_1 + \frac{\beta \mu}{2} \| \mathbf{v}_{j} - \Psi(C^k) \mathbf{e}_{j} - \mathbf{\lambda}_{j}^k \|_2^2.
\end{equation}

%\begin{lemma}[Shrinkage for $\ell_1$ norm]
For any $\lambda, \mu >0$ and $x, y \in \mathbb{R}^n$, the minimizer to
\begin{equation*}
\min_{y} \lambda \|y\|_1 + \frac{\mu}{2} \|y-x\|_2^2
\end{equation*}
is given by
\begin{equation}
y = \mathcal{S}_1 \big(x,\frac{\lambda}{\mu}\big) := \max \big\{|x|-\frac{\lambda}{\mu},0\big\} \odot \mathrm{sgn}(x).
\end{equation}
where $\odot$ denotes  component-wise product and $\mathcal{S}_1$ stands for $\ell_1$-shrinkage.
%\end{lemma}

The closed-form solution to the $V$-subproblem \eqref{eq:vprob} is
\begin{equation}
\mathbf{v}^{k+1}_{j} = \mathcal{S}_1 \big( \Psi(C^k) \mathbf{e}_{j} + \mathbf{\lambda}^k_{j}, \frac{1}{\mu} \big),\quad j = 1, 2, \hdots, d.
\end{equation}

%\item{Reconstruction Fidelity}
\subsubsection{Reconstruction Fidelity}

The $C$-subproblem, as it involves multiple terms and all the input data, is the most time consuming to solve. Specifically, it is
\begin{equation}
\begin{split}
{C}^{k+1} = \arg \underset{C} {\text{min}} \hspace{2mm} &
  \frac{\alpha \mu}{2} \sum_{i=1}^{n} \| \mathbf{u}_{(i)}^{k+1} - \mathbf{e}_{(i)} \Psi(C) - \mathbf{\kappa}_{(i)}^k \|_2^2
+ \frac{\beta \mu}{2} \sum_{j=1}^{d} \| \mathbf{v}_{j}^{k+1} - \Psi(C) \mathbf{e}_{j} - \mathbf{\lambda}_{j}^k \|_2^2 \nonumber \\
& + \sum_{t=1}^{l} \frac{1}{2} \| \mathbf{z}_t - \Phi_t \mathcal{F} (C \hat{\mathbf{x}}_t) \|_2^2.
\end{split}
\end{equation}
Once we further write out the fidelity term, and utilizing the linearity property of the discrete Fourier transform, we have
\begin{equation}
\begin{split}
H(C) & :=  \sum_{t=1}^{l} \frac{1}{2} \| \mathbf{z}_t - \Phi_t \mathcal{F} (C \hat{\mathbf{x}}_t) \|_2^2 \nonumber \\
        & =  \sum_{t=1}^{l} \frac{1}{2} \| \mathbf{z}_t - \Phi_t \mathcal{F} \sum_{j=1}^{d} \mathbf{c}_{j} \hat{x}_{t,j} \|_2^2 \nonumber \\
        & =  \sum_{t=1}^{l} \frac{1}{2} \| \mathbf{z}_t - \sum_{j=1}^{d} \hat{x}_{t,j} \Phi_t \mathcal{F} \mathbf{c}_{j} \|_2^2.
\end{split}
\end{equation}
This leads to the following equivalent problem:
\begin{equation}
\begin{split}
{C}^{k+1} = \arg \underset{C} {\text{min}} \hspace{2mm} &
   \frac{\alpha \mu}{2} \sum_{i=1}^{n} \| \mathbf{u}_{(i)}^{k+1} - \mathbf{e}_{(i)} \Psi(C) - \mathbf{\kappa}_{(i)}^k \|_2^2
+  \frac{\beta \mu}{2} \sum_{j=1}^{d} \| \mathbf{v}_{j}^{k+1} - \Psi(C) \mathbf{e}_{j} - \mathbf{\lambda}_{j}^k \|_2^2 \\
& + \sum_{t=1}^{l} \frac{1}{2} \| \mathbf{z}_t - \sum_{j=1}^{d} \hat{x}_{t,j} \Phi_t \mathcal{F} \mathbf{c}_{j} \|_2^2.
\end{split}
\end{equation}
In the $C$-subproblem, the first term operates in the row space while the second and third terms operate in the column space of the observation matrix $C$; this is undesirable computationally.  However, it is easy to see
\begin{equation}
\sum_{i=1}^{n} \| \mathbf{u}_{(i)}^{k+1} - \mathbf{e}_{(i)} \Psi(C) - \mathbf{\kappa}_{(i)}^k \|_2^2 
= \| U^{k+1} - \Psi(C) - \Upsilon^{k} \|_F^2
= \sum_{j=1}^{d} \| \mathbf{u}_{j}^{k+1} - \Psi(C) \mathbf{e}_{j} - \mathbf{\kappa}_{j}^k \|_2^2, \nonumber
\end{equation}
which allows us to rewrite the $C$-subproblem as follows:
\begin{equation}
\label{eq:cprob}
\begin{split}
C^{k+1} = \arg \underset{C} {\text{min}} \hspace{2mm} &
\frac{\alpha \mu}{2} \sum_{j=1}^{d} \| \mathbf{u}_{j}^{k+1} - \Psi \mathbf{c}_{j} - \mathbf{\kappa}_{j}^k \|_2^2
+ \frac{\beta \mu}{2} \sum_{j=1}^{d} \| \mathbf{v}_{j}^{k+1} - \Psi \mathbf{c}_{j} - \mathbf{\lambda}_{j}^k \|_2^2 \\
& + \sum_{t=1}^{l} \frac{1}{2} \| \mathbf{z}_t - \sum_{j=1}^{d} \hat{x}_{t,j} \Phi_t \mathcal{F} \mathbf{c}_{j} \|_2^2.
\end{split}
\end{equation}
We rewrite the above objective function as
\begin{equation}
\begin{split}
& \frac{\alpha \mu}{2} \sum_{j=1}^{d} (\Psi \mathbf{c}_j)^{\top} (\Psi \mathbf{c}_j) - 2 (\Psi \mathbf{c}_j)^{\top} (\mathbf{u}^{k+1}_j-\mathbf{\kappa}^{k}_j) + (\mathbf{u}^{k+1}_j-\mathbf{\kappa}^{k}_j)^{\top}(\mathbf{u}^{k+1}_j-\mathbf{\kappa}^{k}_j) \\
+ & \frac{\beta \mu}{2} \sum_{j=1}^{d} (\Psi \mathbf{c}_{j})^{\top}(\Psi \mathbf{c}_{j}) - 2 (\Psi \mathbf{c}_{j})^{\top} (\mathbf{v}^{k+1}_j-\mathbf{\lambda}^{k}_j) + (\mathbf{v}^{k+1}_j-\mathbf{\lambda}^{k}_j)^{\top}(\mathbf{v}^{k+1}_j-\mathbf{\lambda}^{k}_j) \\
+ & \frac{1}{2} \sum_{t=1}^{l} \sum_{j=1}^{d} (\hat{x}_{t,j} \Phi_t \mathcal{F} \mathbf{c}_j)^{\top}(\hat{x}_{t,j} \Phi_t \mathcal{F} \mathbf{c}_j) 
   + \frac{1}{2} \sum_{t=1}^{l} \sum_{j \neq j'} (\hat{x}_{t,j} \Phi_t \mathcal{F} \mathbf{c}_j)^{\top}(\hat{x}_{t,j'} \Phi_t \mathcal{F} \mathbf{c}_{j'}) \\
- & \sum_{t=1}^{l} \sum_{i=1}^{d} (\hat{x}_{t,j} \Phi_t \mathcal{F} \mathbf{c}_j)^{\top} \mathbf{z}_t + \frac{1}{2} \sum_{t=1}^{l} \mathbf{z}_t^{\top} \mathbf{z}_t.   
\end{split}
\end{equation}
By taking the first derivative of the objective function and setting it to zero, we derive the normal equation to problem \eqref{eq:cprob}.  Note we use $\Psi$ to denote the wavelet operator, and $\Psi^{\dagger}$ to denote its adjoint operator.  We use $\Phi_t$ to denote the row selector operator, and $\Phi_t^{\dagger}$ to denote its adjoint operator. Similarly, we use $\mathcal{F}$ to denote the Fourier operator, and $\mathcal{F}^{\dagger}$ to denote its adjoint operator.  The normal equation is as follows:
\begin{equation}
\label{eq:normal}
\begin{split}
& \frac{\alpha \mu}{2}
\begin{pmatrix}
{\Psi}^{\dagger} {\Psi} & 0 & \cdots & 0 \\
0 & {\Psi}^{\dagger} {\Psi} & \cdots & 0 \\
\vdots & \vdots & \ddots & \vdots \\
0 & 0 & \cdots & {\Psi}^{\dagger} {\Psi}
\end{pmatrix} 
\begin{pmatrix}
\mathbf{c}_1 \\
\mathbf{c}_2 \\
\vdots \\
\mathbf{c}_d
\end{pmatrix}
+ \frac{\beta \mu}{2}
\begin{pmatrix}
{\Psi}^{\dagger} {\Psi} & 0 & \cdots & 0 \\
0 & {\Psi}^{\dagger} {\Psi} & \cdots & 0 \\
\vdots & \vdots & \ddots & \vdots \\
0 & 0 & \cdots & {\Psi}^{\dagger} {\Psi}
\end{pmatrix}
\begin{pmatrix}
\mathbf{c}_1 \\
\mathbf{c}_2 \\
\vdots \\
\mathbf{c}_d
\end{pmatrix}
\\
+ & \frac{1}{2}
\begin{pmatrix}
\sum_{t} \hat{x}_{t,1}^2 (\Phi_t \mathcal{F})^{\dagger}\Phi_t \mathcal{F} & \sum_{t} \hat{x}_{t,1} \hat{x}_{t,2} (\Phi_t \mathcal{F})^{\dagger} \Phi_t \mathcal{F} & \cdots & \sum_{t} \hat{x}_{t,1} \hat{x}_{t,d} (\Phi_t \mathcal{F})^{\dagger} \Phi_t \mathcal{F} \\
\sum_{t} \hat{x}_{t,2} \hat{x}_{t,1} (\Phi_t \mathcal{F})^{\dagger} \Phi_t \mathcal{F} & \sum_{t} \hat{x}_{t,2}^2 (\Phi_t \mathcal{F})^{\dagger} \Phi_t \mathcal{F} & \cdots & \sum_{t} \hat{x}_{t,2} \hat{x}_{t,d} (\Phi_t \mathcal{F})^{\dagger} \Phi_t \mathcal{F} \\
\vdots & \vdots & \ddots & \vdots \\
\sum_{t} \hat{x}_{t,d} \hat{x}_{t,1} (\Phi_t \mathcal{F})^{\dagger} \Phi_t \mathcal{F} & \sum_{t} \hat{x}_{t,d} \hat{x}_{t,2} (\Phi_t \mathcal{F})^{\dagger} \Phi_t \mathcal{F} & \cdots & \sum_{t} \hat{x}_{t,d}^2 (\Phi_t \mathcal{F})^{\dagger} \Phi_t \mathcal{F} 
\end{pmatrix} 
\begin{pmatrix}
\mathbf{c}_1 \\
\mathbf{c}_2 \\
\vdots \\
\mathbf{c}_d
\end{pmatrix}
\\
= & 
\alpha \mu
\begin{pmatrix}
\Psi^{\dagger} (\mathbf{u}^{k+1}_1 - \mathbf{\kappa}^{k}_1) \\
\Psi^{\dagger} (\mathbf{u}^{k+1}_2 - \mathbf{\kappa}^{k}_2) \\
\vdots \\
\Psi^{\dagger} (\mathbf{u}^{k+1}_d - \mathbf{\kappa}^{k}_d)
\end{pmatrix}
+
\beta \mu
\begin{pmatrix}
\Psi^{\dagger} (\mathbf{v}^{k+1}_1 - \mathbf{\lambda}^{k}_1) \\
\Psi^{\dagger} (\mathbf{v}^{k+1}_2 - \mathbf{\lambda}^{k}_2) \\
\vdots \\
\Psi^{\dagger} (\mathbf{v}^{k+1}_d - \mathbf{\lambda}^{k}_d)
\end{pmatrix}
+ 
\begin{pmatrix}
\sum_{t} \hat{x}_{t,1} (\Phi_t \mathcal{F})^{\dagger} \mathbf{z}_t \\
\sum_{t} \hat{x}_{t,2} (\Phi_t \mathcal{F})^{\dagger} \mathbf{z}_t \\
\vdots \\
\sum_{t} \hat{x}_{t,d} (\Phi_t \mathcal{F})^{\dagger} \mathbf{z}_t \\
\end{pmatrix}.
\end{split}
\end{equation}
We simplify the notation of the normal equation \eqref{eq:normal} as
\begin{equation*}
\mathcal{LHS}  
\begin{pmatrix}
\mathbf{c}_1 \\
\mathbf{c}_2 \\
\vdots \\
\mathbf{c}_d
\end{pmatrix} 
= \mathcal{RHS},
\end{equation*}
and one can immediately notice that the normal equation \eqref{eq:normal} is not a diagonal system.  In other words, $\mathbf{c}_j$'s are coupled.  Solving for linear system \eqref{eq:normal} directly can be computationally expensive.

We use the prox-linear method \cite{chen1994} to decouple the system.  In stead of solving for Eqn.~\eqref{eq:cprob} directly, we solve for the following problem using the prox-linear method that decouples all the $\mathbf{c}_j$'s:
\begin{equation}
\mathbf{c}_j^{k+1} = \arg \min_{\mathbf{c}_j} q(\mathbf{c}_j^k)^{\top} (\mathbf{c}_j-\mathbf{c}_j^k) + \frac{1}{2\delta} \| \mathbf{c}_j-\mathbf{c}_j^k \|_2^2,
\end{equation}
where $q(\mathbf{c}_j^k) = \nabla_C \mathcal{L}(U,V,C)$.  This allows us to solve for each $\mathbf{c}_j$ using block coordinate descent in the Jacobian fashion,
\begin{equation}
\label{eq:c-update}
\mathbf{c}_j^{k+1} = \mathbf{c}_j^{k} - \delta q(\mathbf{c}_j^k).
\end{equation}
More careful inspection on $q(\mathbf{c}_j^k)$ reveals
\begin{equation}
\begin{split}
q(\mathbf{c}_j^k) =
& \hspace{1mm} 2 \alpha \mu {\Psi}^{\dagger} \big({\Psi} \mathbf{c}_j^k - (\mathbf{u}_j^{k+1}-\kappa_j^k)\big) + 2 \beta \mu {\Psi}^{\dagger} \big({\Psi} \mathbf{c}_j^k -(\mathbf{v}_j^{k+1}-\lambda_j^k) \big) \\
& + \sum_t \hat{x}_{t,j}^2 ({\Phi}_t \mathcal{F})^{\dagger} {\Phi}_t \mathcal{F} \mathbf{c}_j^k - \sum_{t} \hat{x}_{t,j} (\Phi_t \mathcal{F})^{\dagger} \mathbf{z}_t.
\end{split}
\end{equation}

\subsubsection{Theoretical Convergence}

We now establish the convergence of Algorithm 2.  We first rewrite the objective function of joint structured sparsity, in the unconstrained optimization form
\begin{equation*}
\min_{C} \hspace{2mm} \alpha \sum_{i=1}^{n} \| \mathbf{e}_{(i)} \Psi(C) \|_2 + \beta \sum_{j=1}^{d} \| \Psi(C) \mathbf{e}_j \|_1
+ \sum_{t=1}^{l} \frac{1}{2} \| \mathbf{z}_t - \Phi_t \mathcal{F} (C \hat{\mathbf{x}}_t) \|_2^2,
\end{equation*}
as the constrained optimization
\begin{subequations}
\label{eq:copt}
\begin{align}
\min_{B,C} \hspace{2mm} & \alpha \sum_{i=1}^{n} \| \mathbf{e}_{(i)} \Psi(B) \|_2 + \beta \sum_{j=1}^{d} \| \Psi(B) \mathbf{e}_j \|_1
+ \sum_{t=1}^{l} \frac{1}{2} \| \mathbf{z}_t - \Phi_t \mathcal{F} (C \hat{\mathbf{x}}_t) \|_2^2 \nonumber \\
\text{s.t.\hspace{2mm}} & B - C = 0. \nonumber
\end{align}
\end{subequations}
We can further group the first two terms in the above constrained optimization together,
\begin{subequations}
\begin{align}
\min_{B,C} \hspace{2mm} & f(B) + g(C) \nonumber \\
\text{s.t.\hspace{2mm}} & B - C = \mathbf{0}, \nonumber
\end{align}
\end{subequations}
with $f(B) = \alpha \sum_{i=1}^{n} \| \mathbf{e}_{(i)} \Psi(B) \|_2 + \beta \sum_{j=1}^{d} \| \Psi(B) \mathbf{e}_j \|_1$ and $g(C) = \sum_{t=1}^{l} \frac{1}{2} \| \mathbf{z}_t - \Phi_t \mathcal{F} (C \hat{\mathbf{x}}_t) \|_2^2$.

Consider the augmented Lagrangian function,
\begin{equation}
\mathcal{L}(B,C,\lambda) = f(B) + g(C) + \frac{\mu}{2} \| B-C-\Lambda \|^2_F,
\end{equation}
where $\Lambda \in \mathbb{R}^{n \times d}$ is the scaled Lagrangian multiplier and $\mu > 0$ is a penalty parameter.  With the above reformulation, we can consolidate Algorithm 2 into a simplified version, see Algorithm 4.  Global and linear convergence for generalized ADMM was analyzed in \cite{deng2012} for constrained convex optimization problems.  We extend those theoretical results from vector case to matrix case below.

\begin{algorithm}[t]
\caption{ADMM for Constrained Optimization}
\label{alg4}
\begin{algorithmic}
\STATE Initialize $B \in \mathbb{R}^{n \times d}$, $C \in \mathbb{R}^{n \times d}$, $k = 0$.
\WHILE {convergence criteria not met}

\STATE \textbf{B-subproblem:} \\
\STATE $B^{k+1} \leftarrow \underset{B} {\text{min}} \hspace{2mm} \mathcal{L}(B, C^k, \lambda^k)$

\STATE \textbf{C-subproblem:} \\
\STATE $C^{k+1} \leftarrow \underset{C} {\text{min}} \hspace{2mm} \mathcal{L}(B^{k+1}, C, \lambda^k)$

\STATE \textbf{Multiplier update:} \\
\STATE $\Lambda^{k+1} \leftarrow \underset{\Lambda} {\text{min}} \hspace{2mm} \Lambda^{k} - \gamma (B^{k+1}-C^{k+1})$

\STATE $k = k+1$
\ENDWHILE
\end{algorithmic}
\end{algorithm}

\begin{theorem}[Global Convergence]
The sequence $\{W^k\} := \{B^k, C^k, \Lambda^k\}$ generated by Algorithm~\ref{alg4} is guaranteed to be bounded.  Moreover, if we assume there exists a saddle point $W^* := (B^*, C^*, \Lambda^*)$ to problem \eqref{eq:copt}, and Hessian $H_g = \nabla^2 g$ satisfies the following condition,
\begin{equation}
\frac{\max(\alpha \mu, \beta \mu)}{\frac{1}{\mu} - \| H_g \|} + \gamma < 2,
\end{equation}
then the sequence $\{W^{k}\}$ generated by Algorithm 3 converges to a KKT point of $W^*$,
\begin{equation*}
\lim_{k \rightarrow \infty} \| W^k - W^* \|_{F}^2 = 0.
\end{equation*}
\end{theorem}
\begin{proof}
%Since both functions $f$ and $g$ are convex, from \cite{glowinski1984}, we know that $\{W^k\} := \{B^k, C^k, \Lambda^k\}$ is guaranteed to be bounded.
First, we assume there exists a saddle point $W^*$, which means $(B^*,C^*,\Lambda^*)$ satisfies the KKT conditions of problem \eqref{eq:copt}:
\begin{equation*}
\Lambda^* \in \partial f(B^*), \hspace{5mm} \Lambda^* \in \partial g(C^*), \hspace{5mm} B^*-C^*=\mathbf{0}.
\end{equation*}

Second, we verify that both functions $f(\cdot)$ and $g(\cdot)$ are convex.  We have $f(B) = \alpha \sum_{i=1}^{n} \| \mathbf{e}_{(i)} \Psi(B) \|_2 + \beta \sum_{j=1}^{d} \| \Psi(B) \mathbf{e}_j \|_1$ and $g(C) = \sum_{t=1}^{l} \frac{1}{2} \| \mathbf{z}_t - \Phi_t \mathcal{F} (C \hat{\mathbf{x}}_t) \|_2^2$.  Since $\| \cdot \|_p$ is convex when $p \ge 1$, and the fact sum of convex functions is also convex, we can easily verify that both functions are convex.

Third, based on Theorem 2.3 remark 3 condition (i) in \cite{deng2012}, we know that the sequence $\{W^k\}$ is bounded. 

Having obtained these assumptions, it follows from Theorem 2.3 in \cite{deng2012} that $\{W^k\}$ has a converging subsequence $\{W^{k'}\}$, whose limit is $W^{*} := \lim_{k' \rightarrow \infty} W^{k'}$.  Hence we have global convergence.
\end{proof}

%\begin{theorem}[Global Linear Convergence]
%The sequence generated by Algorithm 4 $\{W^k\} := \{B^k,C^k,\Lambda^k\}$ has global linear convergence, moreover the convergence rate is R-linear.
%\end{theorem}

%Theorem 4.5 is an immediate result of Theorem 3.6 in \cite{deng2012}. In our case, $P = 0$ and $Q = 0$, scenario 1 in Table 1.2 \cite{deng2012}.

\section{Dynamic MRI Reconstruction Quality}

\subsection{Impact of Sampling Strategies}

We now apply our algorithm to accelerate the acquisition process of dynamic MRI.  Since the sampling strategy in the k-space has an impact on the reconstruction quality, we test three types of sampling strategies following the work of \cite{krahmer2012}.

\begin{figure}[H]
\begin{center}
\includegraphics[width = 0.98\textwidth]{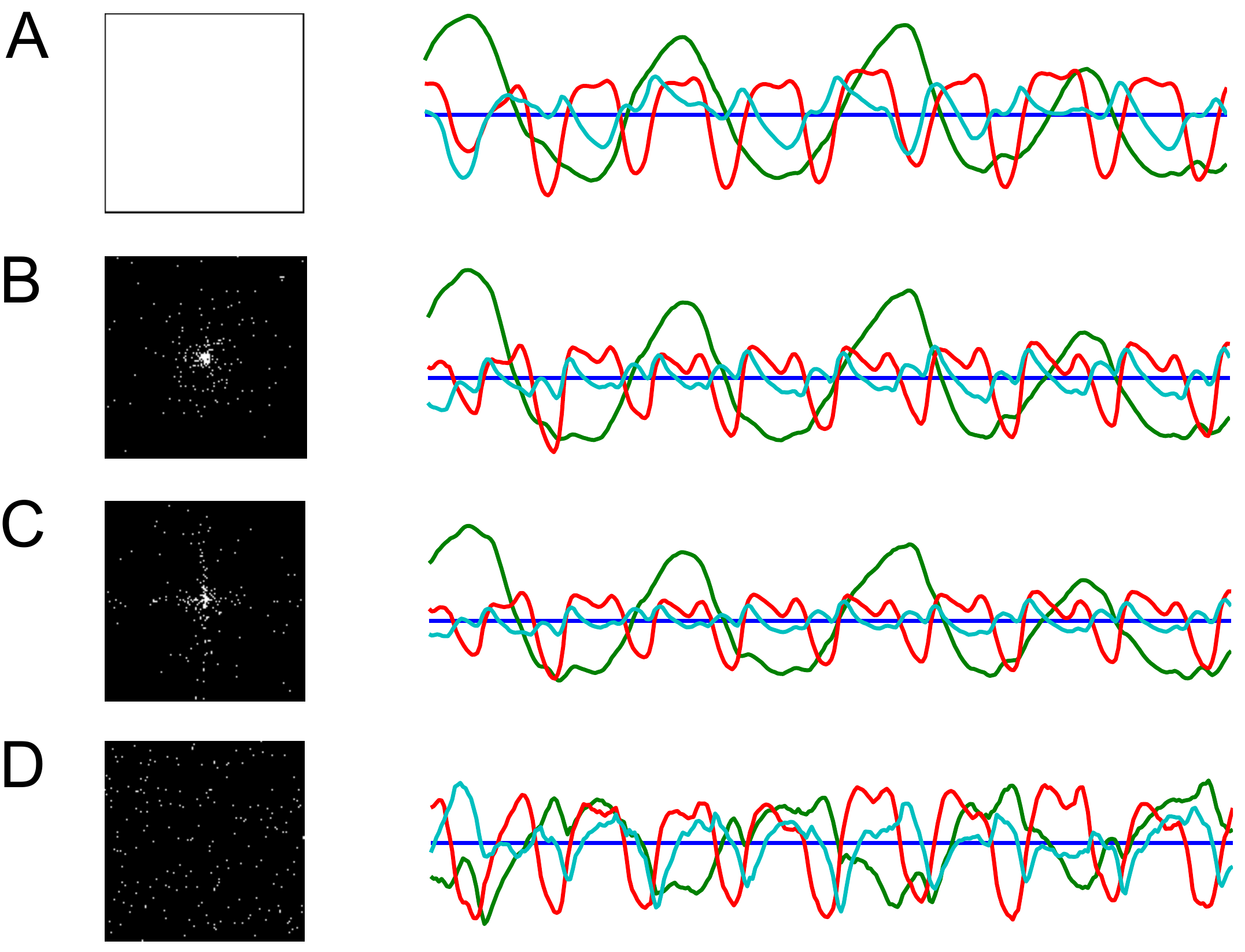}
\end{center}
\caption{Impact of measurement matrix on the state sequence estimate, sampling distribution.  Illustrated are four measurement matrices design for the time-invariant component $\bar{\Phi}$, with all compressive sensing cases using 200 samples: (A) measurement matrix covering all samples in k-space, (B) measurement matrix following \emph{distance} distribution, (C) measurement matrix with \emph{hyperbolic} distribution, (D) measurement matrix with \emph{uniform} distribution.  Left column shows the measurement matrix, right column shows the estimated state sequence using Algorithm 1.}
\label{fig:heartstate}
\end{figure}

We first illustrate the impact of measurement matrix on the state sequence estimate.  Figure~\ref{fig:heartstate} shows three measurement matrices that cover all range of frequencies, however follow different probability distributions:  
\begin{itemize}
\item{\emph{distance}:}\\
probability of sampling falls over as inverse of squared distance to the k-space center.
\item{\emph{hyperbolic}:}\\
probability of sampling falls over as a hyperbolic function in the k-space.
\item{\emph{uniform}:}\\
probability of sampling is uniform in the k-space.
\end{itemize}

\begin{figure}[H]
\begin{center}
Sampling Strategy = \emph{distance} \hspace{5mm} SNR = 19.1 dB \\
\vspace{5mm}
\includegraphics[width = 0.98\textwidth]{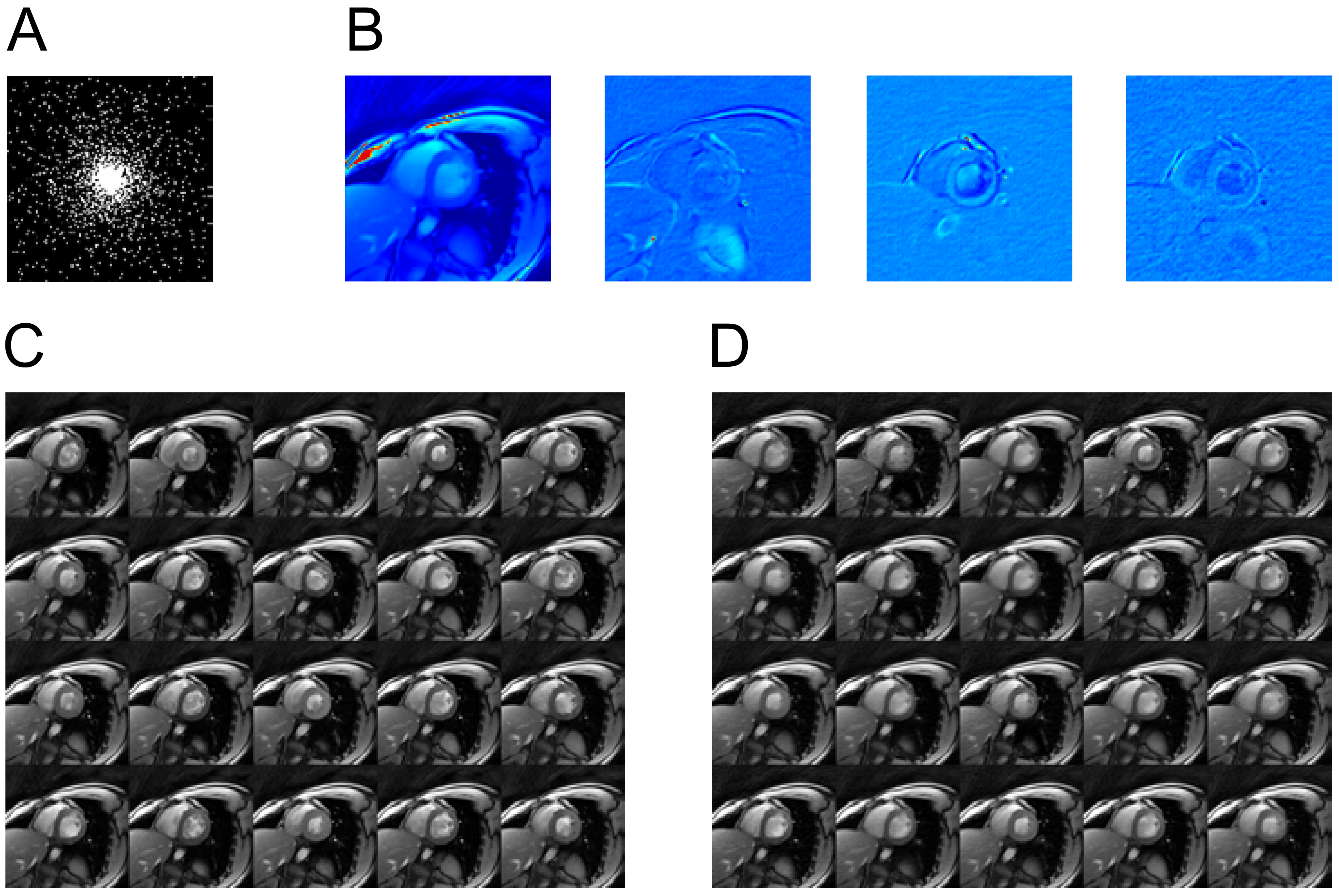}
\end{center}
\caption{Reconstruction result for dynamic MRI data, based on $10\%$ k-space data.  Sampling strategy adopts the \emph{distance} distribution. Shown are (A) measurement matrix $\Phi_t$ at $t=1$, (B) observation matrix $C$ with $d = 4$, (C) original video frames $\mathbf{y}_t$ (D) reconstructed video frames $\hat{\mathbf{y}}_t$.  The reconstruction SNR is 19.1 dB.}
\label{fig:heartdistance}
\end{figure}

Our numerical results indicate the best measurement matrix design is to sample k-space according to the \emph{distance} strategy, where one samples the k-space in a density that falls off as 1 over the squared distance to the center of k-space.

We show the reconstruction quality of dynamic MRI using different sampling strategies.  The cardiac MRI dataset used in this experiment was described in \cite{zhang2010}.  We obtained the reconstructed video for real-time MRI of a human heart, whose spatial resolution is subsampled at 128$\times$128 and temporal resolution is 33 ms, with 300 frames in total.

Note we only simulated a single coil with a homogeneous coil sensitivity map.  We simulated k-t data by taking the Fourier transform and performing subsampling.  Define the samples in Fourier space as
\begin{equation}
\Omega = \{(\omega_1^k, \omega_2^k)\}_{k=1}^{m} \subset \{-\frac{n_x}{2}+1,\hdots,\frac{n_x}{2}, -\frac{n_y}{2}+1,\hdots,\frac{n_y}{2}\}
\end{equation}
assuming a uniform Cartesian grid.

\begin{figure}[H]
\begin{center}
Sampling Strategy = \emph{hyperbolic} \hspace{5mm} SNR = 15.3 dB \\
\vspace{5mm}
\includegraphics[width = 0.98\textwidth]{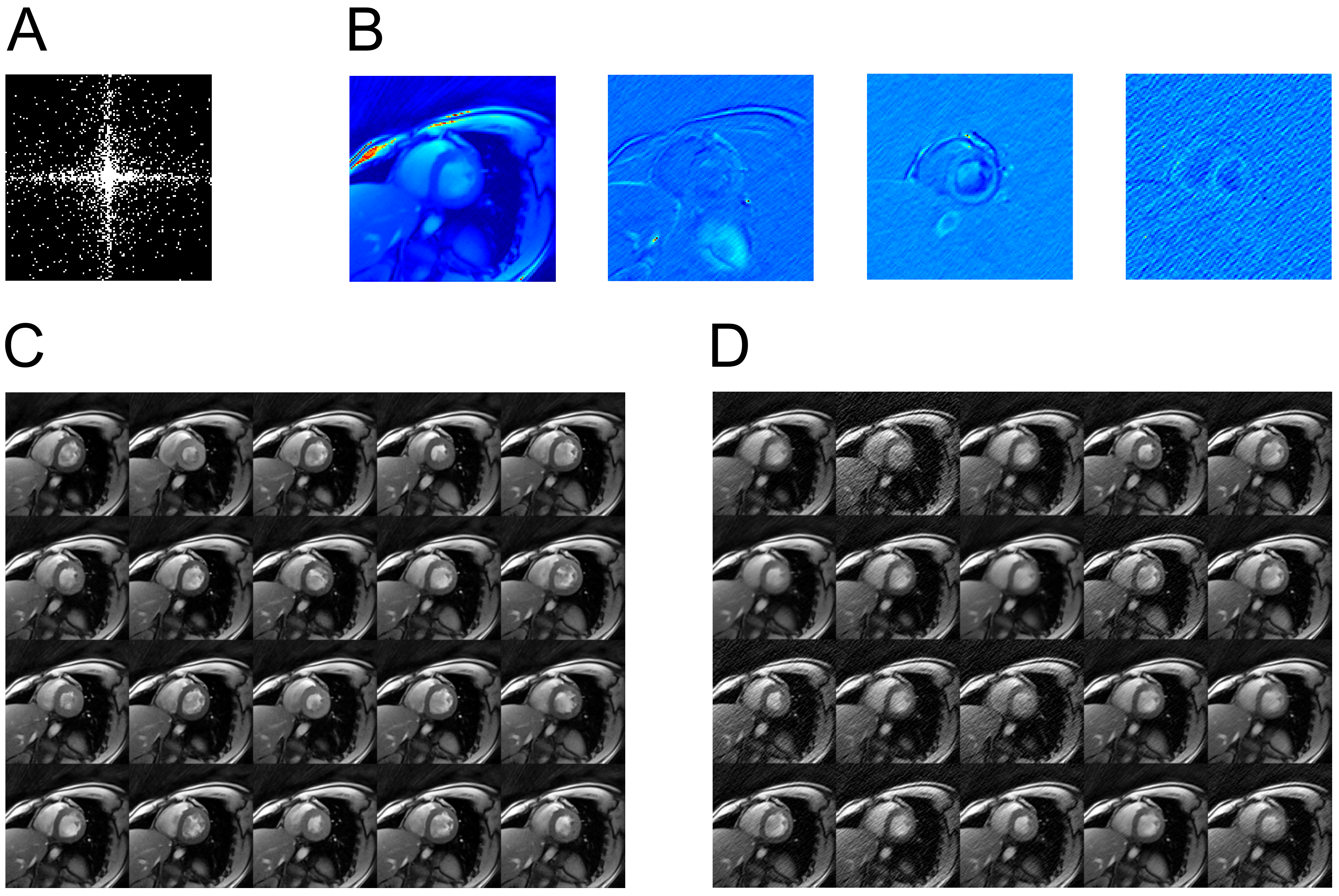}
\end{center}
\caption{Reconstruction result for dynamic MRI data, based on $10\%$ k-space data. Sampling strategy adopts the \emph{hyperbolic} distribution. Shown are (A) measurement matrix $\Phi_t$ at $t=1$, (B) observation matrix $C$ with $d = 4$, (C) original video frames $\mathbf{y}_t$ (D) reconstructed video frames $\hat{\mathbf{y}}_t$.  The reconstruction SNR is 15.3 dB.}
\label{fig:hearthyperbolic}
\end{figure}

We show the reconstruction result for dynamic MRI with $10\%$ k-t data, using the \emph{distance} sampling strategy in Figure~\ref{fig:heartdistance}.  We construct $\Omega$ by subsampling the Fourier space i.i.d.\ according to density
\begin{equation}
\eta(\omega_1,\omega_2) \propto (\omega_1^2 + \omega_2^2+1)^{-1}.
\end{equation}
We attain a SNR of 19.1 dB in the reconstruction using the \emph{distance} sampling strategy.

We show the reconstruction result for the dynamic MRI with $10\%$ k-t data, using the \emph{hyperbolic} sampling strategy in Figure~\ref{fig:hearthyperbolic}.  We construct $\Omega$ by subsampling the Fourier space i.i.d.\ according to density
\begin{equation}
\eta(\omega_1,\omega_2) \propto (\omega_1^2 + \omega_2^2+1)^{-3/2}.
\end{equation}
We attain a SNR of 15.3 dB in the reconstruction using the \emph{hyperbolic} sampling strategy.

\newpage

\begin{figure}[H]
\begin{center}
Sampling Strategy = \emph{uniform} \hspace{5mm} SNR = 7.1 dB \\
\vspace{5mm}
\includegraphics[width = 0.98\textwidth]{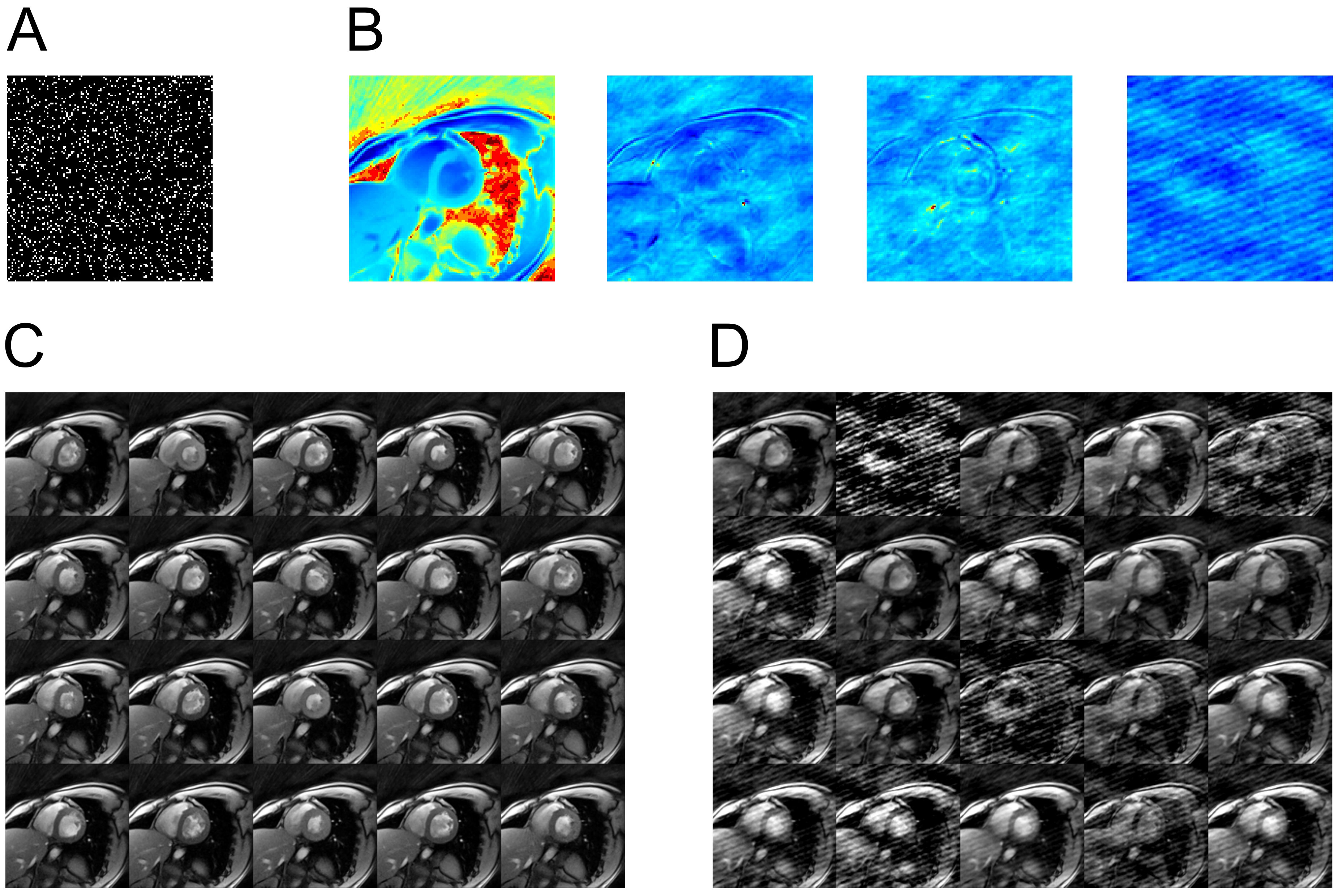}
\end{center}
\caption{Reconstruction result for dynamic MRI data, based on $10\%$ k-space data. Sampling strategy adopts the \emph{uniform} distribution. Shown are (A) measurement matrix $\Phi_t$ at $t=1$, (B) observation matrix $C$ with $d = 4$, (C) original video frames $\mathbf{y}_t$ (D) reconstructed video frames $\hat{\mathbf{y}}_t$. The reconstruction SNR is 7.1 dB.}
\label{fig:heartuniform}
\end{figure}

We show the reconstruction result for the dynamic MRI with $10\%$ k-t data, using the \emph{uniform} sampling strategy in Figure~\ref{fig:heartuniform}.  We construct $\Omega$ by subsampling the Fourier space i.i.d.\ according to density
\begin{equation}
\eta(\omega_1,\omega_2) \propto 1.
\end{equation}
We attain a SNR of 7.1 dB in the reconstruction using the \emph{uniform} sampling strategy.

\subsection{Comparison with Prior Art}

We compare kt-CSLDS with prior art in the literature, which includes kt-SPARSE, MASTeR, and L+S.  Figure~\ref{fig:heart1comp} and Figure~\ref{fig:heart2comp} show the numerical results on two datasets described in \cite{zhang2010}.  Both datasets can be downloaded from the paper website provided by the authors.  Our numerical results show that kt-CSLDS achieves excellent reconstruction quality.

\newpage

\begin{figure}[H]
\begin{center}
\includegraphics[width = 0.98\textwidth]{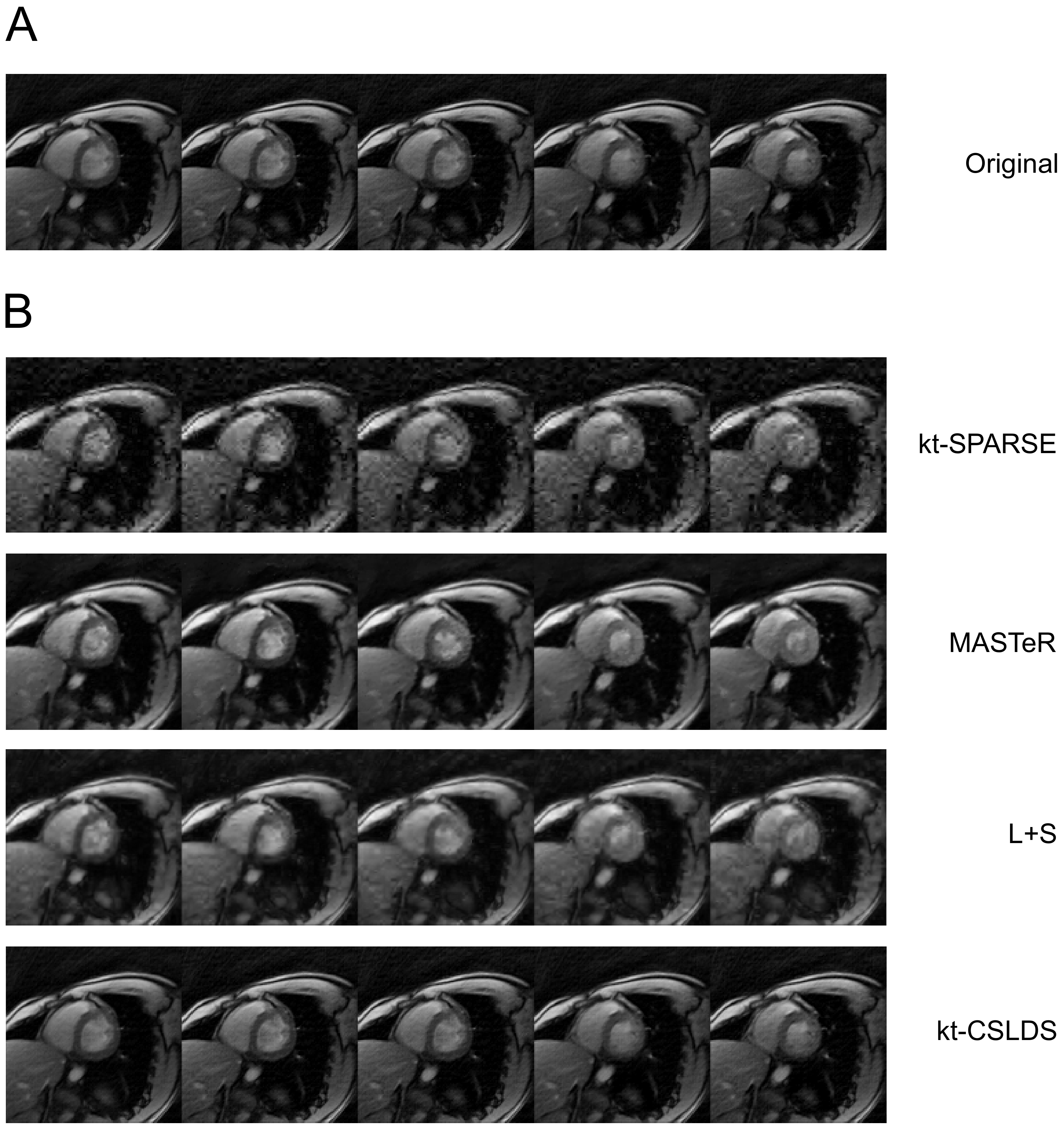}
\end{center}
\caption{Comparison of kt-CSLDS with prior art in the literature.  Numerical results are based on one dataset for dynamic heart imaging, with 1.5 mm resolution, 8 mm section thickness. The original dynamic MRI is acquired at 30 ms acquisition time with 300 frames.  We downsample the heart video to 128$\times$128 spatial resolution and simulate a single coil acquisition.  We use 10$\times$ compression rate for this experiment, and employ the \emph{distance} sampling strategy for compressive measurement.  (A) Sample frames from the original heart video.  (B) Reconstructed frames based on different video compressive sensing algorithms.  Their  respective reconstruction SNRs are as follows: kt-SPARSE (13.0 dB), MASTeR (18.8 dB), L+S (15.8 dB), kt-CSLDS (19.1 dB).}
\label{fig:heart1comp}
\end{figure}

\begin{figure}[H]
\begin{center}
\includegraphics[width = 0.98\textwidth]{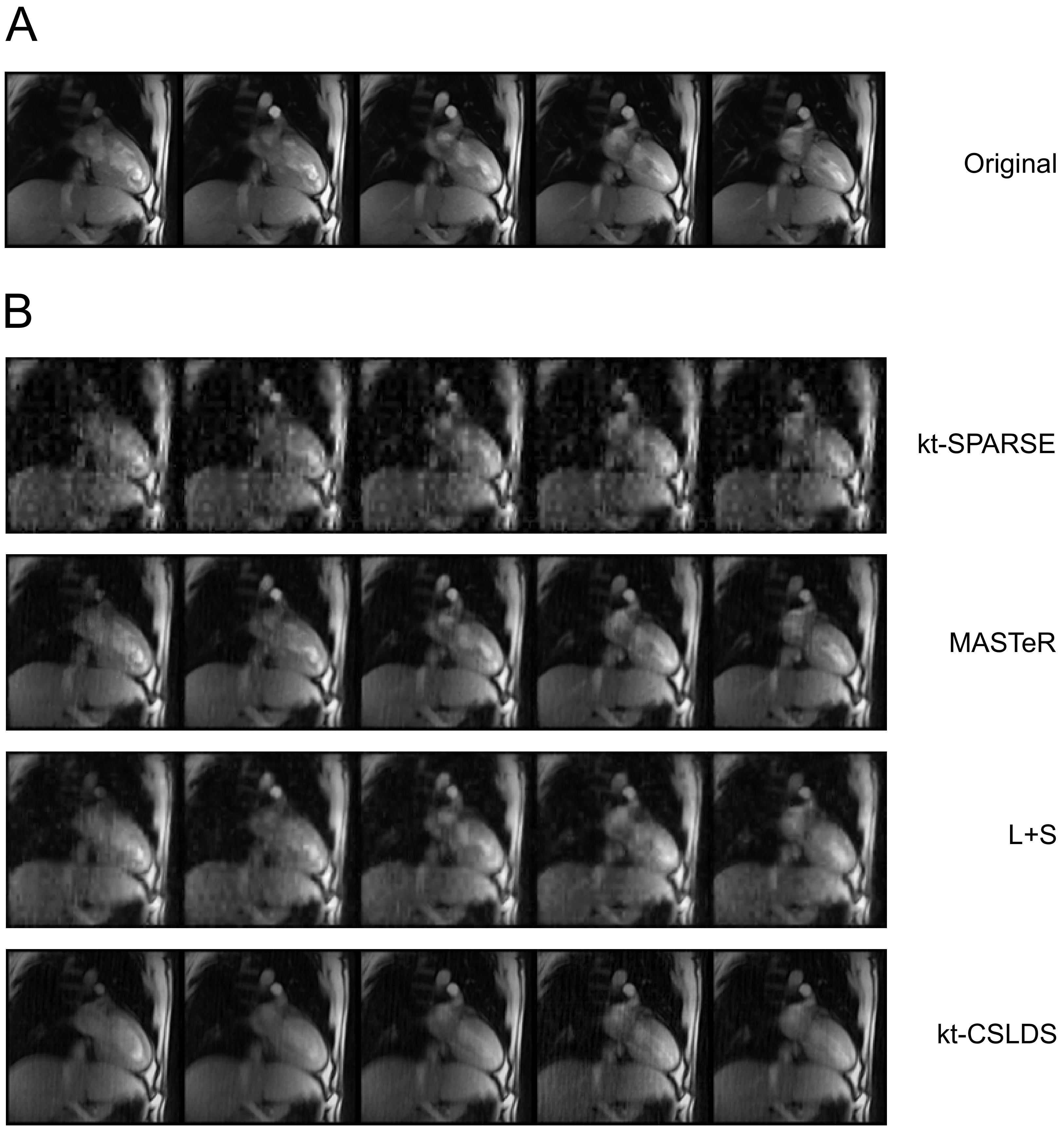}
\end{center}
\caption{Comparison of kt-CSLDS with prior art in the literature.  Numerical results are based on another dataset for dynamic heart imaging, with 2.0 mm resolution, 8 mm section thickness. The original dynamic MRI is acquired at 22 ms acquisition time with 360 frames.  We downsample the heart video to 128$\times$128 spatial resolution and simulate a single coil acquisition.  We use 10$\times$ compression rate for this experiment and employ the \emph{distance} sampling strategy for compressive measurement.  (A) Sample frames from the original heart video.  (B) Reconstructed frames based on different video compressive sensing algorithms.  Their respective reconstruction SNRs are as follows: kt-SPARSE (14.0 dB), MASTeR (19.4 dB), L+S (16.3 dB), kt-CSLDS (20.3 dB).}
\label{fig:heart2comp}
\end{figure}

Table~\ref{tab:snr} compares the reconstruction SNR of different video compressive sensing models under various compression rate.  Table~\ref{tab:time} shows their respective computation time.  In comparison, kt-CSLDS achieves the best reconstruction quality while consuming the least computational time.

\begin{table}
\caption{Comparison of reconstruction SNR}
\begin{center}
\begin{tabular}{cccccc}
\hline\hline
Model & 10$\times$ & 20$\times$ & 30$\times$ & 40$\times$ & 50$\times$ \\
\hline
kt-SPARSE & 13.0 dB & 11.3 dB & 10.5 dB & 9.9 dB & 9.5 dB \\
MASTeR    & 18.8 dB & 14.6 dB & 12.7 dB & 11.6 dB & 11.0 dB \\
L+S            & 15.8 dB & 12.4 dB & 10.7 dB & 9.7 dB & 9.3 dB \\
kt-CSLDS  & 19.1 dB & 16.3 dB & 15.0 dB & 13.3 dB & 12.8 dB  \\
\hline\hline
\end{tabular}
\end{center}
\label{tab:snr}
\end{table}

\begin{table}
\caption{Comparison of computation time}
\begin{center}
\begin{tabular}{cccccc}
\hline\hline
Model & 10$\times$ & 20$\times$ & 30$\times$ & 40$\times$ & 50$\times$ \\
\hline
kt-SPARSE & 371.7 s & 401.9 s & 455.0 s & 491.1 s & 585.2 s \\
MASTeR    & 422.6 s & 430.9 s & 425.7 s & 423.6 s & 433.9 s \\
L+S            & 1390.2 s & 1403.6 s & 1402.8 s & 1406.4 s & 1406.4 s \\
kt-CSLDS   & 7.6 s & 5.0 s & 4.9 s & 4.5 s & 4.4 s \\
\hline\hline
\end{tabular}
\end{center}
\label{tab:time}
\end{table}

\section{Conclusions}

In this paper, we built upon video compressive sensing ideas to accelerate the imaging acquisition process of dynamic MRI.  We extended CS-LDS model to the Fourier-time space, resulting in so-called kt-CSLDS.  Efficient numerical algorithm was derived based on ADMM.  Theoretical analysis was carried out to ensure global convergence.  Numerical results show that kt-CSLDS achieves favorable reconstruction quality while being computationally efficient, in comparison with state-of-the-art dynamic MRI compressive sensing literature.

LDS provides a compact model for video sequences, which approximates high-dimensional signal using low-dimensional representation.  Therefore, kt-CSLDS benefits from such a compact representation, since the number of unknowns are much smaller compared with the original video cube.  This explains why our model achieves high-fidelity reconstruction results given compressive measurements.  The computational speed we gain is a result of both smaller dimensionality of the optimization problem and customized algorithm based on ADMM.  

There are many ways to build upon the current kt-CSLDS framework.  Our current methodology takes all the video data and performs batch process.  Future work seeks online version of the current reconstruction algorithm.  Regarding the measurement strategy, we have shown empirically that the best strategy is to sample the k-space according to the \emph{distance} strategy.  Such a result is consistent with theory for Fourier compressive sensing for static MR imaging \cite{krahmer2012}.  It remains an open theoretical question why such a strategy is optimal for video compressive sensing.

\section*{Acknowledgments}

We thank Rachel Ward for generously providing the code to generate the sampling patterns in the k-space.

\bibliographystyle{plain} % wotao
\bibliography{VideoCSLDSMRIRef}

\end{document}